\documentclass[11pt]{article}

\usepackage{fullpage}
\usepackage{amsmath,amsfonts,amsthm,amssymb}
\usepackage{enumerate}
\usepackage{color}
\usepackage{cite}

\usepackage{microtype}
\usepackage{graphicx}

\usepackage{hyperref}

\usepackage{algorithm}
\usepackage{algorithmic}

\definecolor{violet}{rgb}{0.5,0.0,1.0}

\definecolor{grey}{rgb}{0.7,0.7,0.7}

\usepackage{macros-shapley}

\newcommand*\samethanks[1][\value{footnote}]{\footnotemark[#1]}

\title{A Distributional Framework for Data Valuation}
\author{Amirata Ghorbani\thanks{Equal Contribution. Authors listed alphabetically.}\\\texttt{amiratag@stanford.edu} \and Michael P. Kim\samethanks[1] \thanks{Supported in part by CISPA Center for Information Security and NSF Award IIS-1908774}\\\texttt{mpk@cs.stanford.edu} \and James Zou\\\texttt{jamesz@stanford.edu}}
\date{}

\begin{document}
\maketitle

\begin{abstract}
Shapley value is a classic notion from game theory,
historically used to quantify the contributions of individuals within groups, and more recently applied to assign values to data points when training machine learning models.
Despite its foundational role, a key limitation of the data Shapley framework is that it only provides valuations for points within a \emph{fixed data set}.
It does not account for
statistical aspects of the
data and does not give a way to reason about points outside the data set.

To address these limitations,
we propose a novel framework -- \emph{distributional Shapley} -- where the value of a point is defined in the context of an underlying data distribution.
We prove that distributional Shapley has several desirable statistical properties; for example, the values are stable under perturbations to the data points themselves and to the underlying data distribution.
We leverage these properties to develop a new algorithm for estimating values from data, which comes with formal guarantees and runs two orders of magnitude faster than state-of-the-art algorithms for computing the (non-distributional) data Shapley values.
We apply distributional Shapley to diverse data sets and demonstrate its utility in a data market setting.
\end{abstract}

\section{Introduction}
\label{sec:intro}

As data becomes an essential driver of innovation and service, how to quantify the value of data is an increasingly important topic of inquiry with policy, economic, and machine learning (ML) implications.
In the policy arena, recent proposals, such as the Dashboard Act in the U.S. Senate, stipulate that large companies quantify the value of data they collect.
In the global economy, the business model of many companies involves buying and selling data.
For ML engineering, it is often beneficial to know which type of training data is most valuable and, hence, most deserving of resources towards collection and annotation.
As such, a principled framework for data valuation would be tremendously useful in all of these domains. 

Recent works initiated a formal study of data valuation in ML \cite{datashapley,jia2019towards}.
In a typical setting, a data set $B = \{z_i\}$ is used to train a ML model,
which achieves certain performance, say classification accuracy $0.9$.
The data valuation problem is to
assign credit amongst the training set, so that each point gets an ``equitable'' share for its contribution towards achieving the $0.9$ accuracy.
Most works have focused on leveraging \emph{Shapley value} as the metric to quantify the contribution of individual $z_i$.
The focus on Shapley value
is in large part
due to the fact that
Shapley uniquely satisfies basic properties for equitable credit allocation \cite{shapley1953value}.
Empirical experiments also show that data Shapley is very effective -- more so than leave-one-out scores -- at identifying points whose addition or removal substantially impacts learning \cite{ghorbani2017interpretation,datashapley}.

At a high-level, prior works on data Shapley require three ingredients: (1) a fixed training data set of $m$ points; (2) a learning algorithm; and (3) a performance metric that measures the overall value of a trained model.
The goal of this work is to significantly reduce the dependency on the first ingredient.
While convenient, formulating the value based on a \emph{fixed data set} disregards crucial statistical considerations and, thus, poses significant practical limitations.

In standard settings, we imagine that data is sampled from a distribution $\mathcal{D}$; measuring the Shapley value with respect to a fixed data set
ignores this underlying distribution.
It also means that the value of a data point computed within one data set
may not make sense when the point is transferred to a new data set.
If we actually want to buy and sell data, then it is important that the value of a given data point represents some intrinsic quality of the datum within the distribution.
For example, a data seller might determine that $z$ has high value based on their data set
$B_s$
and sell $z$ to a buyer at a high price.
Even if the buyer's data set 
$B_b$
is drawn from a similar distribution as 
$B_s$,
the existing data Shapley framework
provides no guarantee of consistency between the value of $z$ computed within
$B_s$
and within
$B_b$.
This inconsistency may be especially pronounced in the case when
the buyer has significantly less data than the seller.

\subsection*{Our contributions}
\paragraph{Conceptual.} Extending prior works on data Shapley, we formulate and develop a notion of \emph{distributional Shapley value} in Section~\ref{sec:dshapley}.
We define the distributional variant in terms of the original data Shapley:
the distributional Shapley value is taken to be the expected data Shapley value, where the data set is drawn i.i.d.\ from the underlying data distribution.
Reformulating this notion of value as a statistical quantity allows us to prove that the notion is stable with respect to perturbations to the inputs as well as the underlying data distribution.
Further, we show a mathematical identity that gives an equivalent definition of distributional Shapley as an expected marginal performance increase by adding the point,
suggesting an unbiased estimator.

\textbf{Algorithmic.}~ In Section~\ref{sec:alg}, we develop this estimator into a novel sampling-based algorithm, \textsc{$\D$-Shapley}.
In contrast to prior estimation heuristics, \textsc{$\D$-Shapley} comes with strong formal approximation guarantees.
Leveraging the stability properties of distributional Shapley value and the simple nature of our algorithm, we develop theoretically-principled optimizations to \textsc{$\D$-Shapley}.
In our experiments across diverse tasks, the optimizations lead to order-of-magnitude reductions in computational costs while maintaining the quality of estimations.

\textbf{Empirical.}~ Finally, in Section~\ref{sec:case}, we present a data pricing case study that demonstrates the consistency of values produced by \textsc{$\D$-Shapley}.
In particular, we show that a data broker can list distributional Shapley values as ``prices,'' which a collection of buyers all agree are fair (i.e.\ the data gives each buyer as much value as the seller claims).
In all, our results demonstrate that the distributional Shapley framework represents a significant step towards the practical viability of the Shapley-based approaches to data valuation.

\subsection*{Related works}
Shapley value, introduced in \cite{shapley1953value}, has been studied extensively in the literature on cooperative games and economics \cite{shapley1988shapley}, and has traditionally been used in the valuation of private information and data markets~\cite{kleinberg2001value,agarwal2019marketplace}.

Our work is most directly related to recent works that apply Shapley value to the data valuation problem.
\cite{datashapley} developed the notion of ``Data Shapley'' and provided algorithms to efficiently estimate values.
Specifically, leveraging the permutation-based characterization of Shapley value, they developed a ``truncated Monte Carlo'' sampling scheme (referred to as \textsc{TMC-Shapley}), demonstrating empirical effectiveness across various ML tasks.
\cite{jia2019towards} introduce several additional approximation methods for efficient computation of Shapley values for training data; subsequently, \cite{jia2019efficient} provided an algorithm for exact computation of Shapley values for the specific case of nearest neighbor classifiers.

Beyond data valuation, the Shapley framework has been used in a variety of ML applications, e.g.\ as a measure of feature importance \cite{cohen2007shap4,kononenko2010shap1,datta2016shap2,lundberg2017shap3,chen2018shapley}.
The idea of a distributional Shapley value bears resemblance to the Aumann-Shapley value \cite{aumann}, a measure-theoretic variant of the Shapley that quantifies the value of individuals within a continuous ``infinite game.''
Our distributional Shapley value focuses on the tangible setting of finite data sets drawn from a (possibly continuous) distribution.

\section{Distributional Data Valuation}
\label{sec:dshapley}
\subsection*{Preliminaries}

Let $\D$ denote a data distribution supported on a universe $\Z$.
For supervised learning problems, we often think of $\Z = \X \times \Y$ where $\X \subseteq \R^d$ and $\Y$ is the output, which can be discrete or continuous.
For $m \in \N$, let $S \sim \D^m$ a collection of $k$ data points sampled i.i.d.\ from $\D$.
Throughout, we use the shorthand $[m] = \set{1,\hdots,m}$ and let $k \sim [m]$ denote a uniform random sample from $[m]$.

We denote by $\pot:\Z^* \to [0,1]$ a potential function\footnote{We use $\Z^* = \bigcup_{n\in \N} \Z^n$ to indicates any finite Cartesian product of $\Z$ with itself; thus, $\pot$ is well-defined on the any natural number of inputs from $\Z$.} or performance metric, where for any $S \subseteq \Z$, $\pot(S)$ represents abstractly the value of the subset. While our analysis applies broadly, in our context, we think of $\pot$ as capturing both the \emph{learning algorithm} and the \emph{evaluation metric}.
For instance, in the context of training a logistic regression model, we might think of $\pot(S)$ as returning the population accuracy of the empirical risk minimizer when $S$ is the training set.

\subsection{Distributional Shapley Value}
Our starting point is the data Shapley value, proposed in \cite{datashapley,jia2019towards} as a way to valuate training data equitably.
\begin{definition}[Data Shapley Value]
Given a potential function $\pot$
and data set $B \subseteq \Z$ where $\card{B} = m$, the data Shapley value of a point $z \in B$ is defined as
\begin{gather*}
\sh(z;\pot,B) \triangleq \frac{1}{m} \sum_{k=1}^m \frac{1}{\binom{m-1}{k-1}} \sum_{\substack{S \subseteq B\setminus\set{z}:\\\card{S}=k-1}} \left(\pot(S \cup \set{z}) - \pot(S)\right).
\end{gather*}
\end{definition}
In words, the data Shapley value of a point $z \in B$ is a weighted empirical average over subsets $S \subseteq B$ of the marginal potential contribution of $z$ to each $S$;
the weighting is such that each possible cardinality $\card{S} = k \in \set{0,\hdots,m-1}$ is weighted equally.
The data Shapley value satisfies a number of desirable properties; indeed, it is the unique valuation function that satisfies the Shapley axioms\footnote{For completeness, the axioms -- symmetry, null player, additivity, and efficiency -- are reviewed in Appendix~\ref{app:axioms}.}. Note that
as the data set size grows, the absolute magnitude of individual data points' values typically scales inversely.

While data Shapley value is a natural solution concept for
data valuation,
its formulation leads to several limitations.
In particular, the values may be very sensitive to the exact choice of $B$; given another $B' \neq B$ where $z \in B \cap B'$, the value $\sh(z;\pot,B)$ might be quite different from $\sh(z;\pot,B')$.
At the extreme, if a new point $z' \not \in B$ is added to $B$, then in principle, we would have to rerun the procedure to compute the data Shapley values for all points in $B \cup \set{z'}$.

In settings where our data are drawn from an underlying distribution $\D$,
a natural extension to the data Shapley approach would parameterize the valuation function by $\D$,
rather than the specific draw of the data set.
Such a distributional Shapley value should be more stable, by removing the explicit dependence on the draw of the training data set.

\begin{definition}[Distributional Shapley Value]
Given a potential function $\pot:\Z^* \to [0,1]$, a distribution $\D$ supported on $\Z$, and some $m \in \N$, the distributional Shapley value of a point $z \in \Z$ is the expected data Shapley value over data sets of size $m$ containing $x$.
\begin{equation*}
\val(z;\pot,\D,m) \triangleq \E_{B \sim \D^{m-1}}\left[\sh\left(z;\pot,B \cup \set{z}\right)\right]
\end{equation*}
\end{definition}
In other words, we can think of the data Shapley value as a random variable that depends on the specific draw of data from $\D$.
Taking the distributional Shapley value $\val(z;\pot,\D,m)$ to be the expectation of this random variable eliminates instability caused by the variance of $\sh(z;\pot,B)$.
While distributional Shapley is simple to state based on the original Shapley value, to the best of our knowledge, the concept is novel to this work.

We note that, while more stable, the distributional Shapley value inherits many of the desirable properties of Shapley, including the Shapley axioms and an expected efficiency property; we cover these in Appendix~B.
Importantly, distributional Shapley also has a clean characterization as the expected gain in potential by adding $z \in \Z$ to a random data set (of random size). 
\begin{theorem}
\label{thm:Dshapley}
Fixing $\pot$ and $\D$, for all $z \in \Z$ and $m \in \N$,
\begin{equation*}
\val(z;\pot,\D,m) = \E_{\substack{k \sim [m]\\S \sim \D^{k-1}}}\left[\pot(S \cup \set{z}) - \pot(S)\right]
\end{equation*}
That is, the distributional Shapley value of a point is its expected marginal contribution in $\pot$ to a set of i.i.d.\ samples from $\D$ of uniform random cardinality.
\end{theorem}
\begin{proof}
The identity holds as a consequence of the definition of data Shapley value and linearity of expectation. 
\begin{align*}
\val(z;\pot,\D,m) &= \E_{D \sim \D^{m-1}}\left[\sh(z;\pot,D \cup \set{z})\right]\\
&=\E_{D \sim \D^{m-1}}\left[\frac{1}{m} \sum_{k=1}^m \frac{1}{\binom{m-1}{k-1}} \sum_{\substack{S \subseteq D:\\\card{S}=k-1}} \left(\pot(S \cup \set{z}) - \pot(S)\right)\right]\\
&= \frac{1}{m} \sum_{k=1}^m \frac{1}{\binom{m-1}{k-1}} \E_{D \sim \D^{m-1}}\left[\sum_{\substack{S \subseteq D:\\\card{S}=k-1}} \left(\pot(S \cup \set{z}) - \pot(S)\right)\right]\\
&=\frac{1}{m} \sum_{k=1}^m \E_{S \sim \D^{k-1}} \left[\pot(S \cup \set{z}) - \pot(S)\right]\addtag\label{thm:Dshapley:pf}\\
&=\E_{\substack{k \sim [m]\\S \sim \D^{k-1}}} \left[\pot(S \cup \set{z}) - \pot(S)\right]
\end{align*}
where (\ref{thm:Dshapley:pf}) follows by the fact that $D \sim \D^{m-1}$ consists of i.i.d.\ samples, so each $S \subseteq D$ with $\card{S} = k-1$ is identically distributed according to $\D^{k-1}$.
\end{proof}

\paragraph{Example: mean estimation.} Leveraging this characterization, for well-structured problems, it is possible to give analytic expressions for the distributional Shapley values.
For instance, consider estimating the mean $\mu$ of a distribution $\D$ supported on $\R^d$.
For a finite subset $S \subseteq \R^d$, we take a potential $\pot(S)$ based on the empirical estimator $\hat{\mu}_S$.
\begin{align*}
\pot_\mu(S) &= \E_{s \sim \D}\left[\norm{s - \mu}^2\right] - \norm{\hat{\mu}_S - \mu}^2
\end{align*}
\begin{proposition}
Suppose $\D$ has bounded second moments.
Then for $z \in \Z$ and $m \in \N$, $\val(z;\pot_\mu,\D,m)$ for mean estimation over $\D$ is given by
\begin{gather*}
\frac{\E_{S \sim \D^m}\left[\pot(S)\right]}{m} + 
\frac{C_m}{m}\cdot \left(\E_{s \sim \D}\left[\norm{s-\mu}^2\right] - \norm{z - \mu}^2\right)
\end{gather*}
for an explicit constant $C_m = \Theta(1)$ determined by $m$.
\end{proposition}
Intuitively, this proposition (proved in Appendix~\ref{app:meanest}) highlights some desirable properties of distributional Shapley: the expected value for a random $z \sim \D$ is an uniform share of the potential for a randomly drawn data set $S \sim \D^m$; further, a point has above-average value when it is closer to $\mu$ than expected.
In general, analytically deriving the distributional Shapley value may not be possible. 
In Section~\ref{sec:alg}, we show how the characterization of Theorem~\ref{thm:Dshapley} leads to an efficient algorithm for estimating values.

\subsection{Stability of distributional Shapley values}

Before presenting our algorithm, we discuss stability properties of distributional Shapley, which are interesting in their own right, but also have algorithmic implications.
We show that when the potential function $\pot$ satisfies a natural stability property, the corresponding distributional Shapley value inherits stability under perturbations to the data points and the underlying data distribution.
First, we recall a standard notion of deletion stability, often studied in the context of generalization of learning algorithms \cite{bousquet2002stability}.
\begin{definition}[Deletion Stability]
For potential $\pot:\Z^* \to [0,1]$ and non-increasing $\beta:\N \to [0,1]$, $\pot$ is $\beta(k)$-deletion stable if for all $k \in \N$ and $S \in \Z^{k-1}$, for all $z \in \Z$
\begin{equation*}
\card{\pot(S \cup \set{z}) - \pot(S)} \le \beta(k).
\end{equation*}
\end{definition}
We can similarly discuss the idea of replacement stability, where we bound $\card{\pot(S \cup \set{z}) - \pot(S \cup \set{z'})}$;
note that by the triangle inequality, $\beta(k)$-deletion stability of $\pot$ implies $2\beta(k)$-replacement stability.
To analyze the properties of distributional Shapley, a natural strengthening of replacement stability will be useful, which
we call \emph{Lipschitz stability}.
Lipschitz stability is parameterized by a metric $d$,
requires the degree of robustness under replacement of $z$ with $z'$ to scale according to the distance $d(z,z')$.
\begin{definition}[Lipschitz Stability]
Let $(\Z,d)$ be a metric space.
For potential $\pot:\Z^* \to [0,1]$ and non-increasing $\beta:\N \to [0,1]$, $\pot$ is $\beta(k)$-Lipschitz stable with respect to $d$ if for all $k \in \N$, $S \in \Z^{k-1}$, and all $z,z' \in \Z$,
\begin{equation*}
\card{\pot(S \cup \set{z}) - \pot(S \cup \set{z'})} \le \beta(k) \cdot d(z,z').
\end{equation*}
\end{definition}
By taking $d$ to be the trivial metric, where $d(z,z') = 1$ if $z \neq z'$, we see that Lipschitz-stability generalizes the idea of replacement stability; still, 
there are natural learning algorithms that satisfy Lipschitz stability for nontrivial metrics.
As one example, we show that Regularized empirical risk minimization over a Reproducing Kernel Hilbert Space (RKHS) -- a prototypical example of a replacement stable learning algorithm -- also satisfies this stronger notion of Lipschitz stability.
We include a formal statement and proof in Appendix~\ref{app:rkhs}.

\paragraph{Similar points receive similar values.}
As discussed, a key limitation with the data Shapley approach for fixed data set $B$ is that we can only ascribe values to $z \in B$.
Intuitively, however, we would hope that if two points $z$ and $z'$ are similar according to some appropriate metric, then they would receive similar Shapley values.
We confirm this intuition for distributional Shapley values when the potential function $\pot$ satisfies Lipschitz stability.
\begin{theorem}
\label{thm:stable:points}
Fix a metric space $(\Z,d)$ and a distribution $\D$ over $\Z$; let $\pot:\Z^* \to [0,1]$ be $\beta(k)$-Lipschitz stable with respect to $d$.
Then for all $m \in \N$, for all $z,z' \in \Z$,
\begin{equation*}
\card{\val(z;\pot,\D,m) - \val(z';\pot,\D,m)} \le \E_{k \sim [m]}\left[\beta(k)\right]\cdot d(z,z').
\end{equation*}
\end{theorem}
\begin{proof}
For any data set size $m \in \N$, we expand $\val(z';\pot,\D,m)$ to express it in terms of $\val(z;\pot,\D,m)$. 
\begin{align*}
\val(z';\pot,\D,m)
&=
\E_{\substack{k \sim [m]\\S \sim \D^{k-1}}}\left[U(S \cup \set{z'}) - U(S)\right]\\
&= \E_{\substack{k \sim [m]\\S \sim \D^{k-1}}}\left[U(S \cup \set{z}) - U(S)\right] + \E_{\substack{k \sim [m]\\S \sim \D^{k-1}}}\left[U(S \cup \set{z'}) - U(S \cup \set{z})\right]\\
&\le \val(z;\pot,\D,m) + \E_{k \sim [m]}\left[\beta(k)\right]\cdot d(z,z')\addtag \label{thm:eqn:lipx}
\end{align*}
where (\ref{thm:eqn:lipx}) follows by the assumption that $\pot$ is $\beta(k)$-Lipschitz stable and linearity of expectation.
\end{proof}
Theorem~\ref{thm:stable:points} suggests that in many settings of interest, the distributional Shapley value will be Lipschitz in $z$.
This Lipschitz property also suggests that, given the values of a (sufficiently-diverse) set of points $Z$, we may be able to infer the values of unseen points $z' \not \in Z$ through interpolation.
Concretely, in Section~\ref{sec:alg:speedup}, we leverage this observation to give an order of magnitude speedup over our baseline estimation algorithm.

\paragraph{Similar distributions yield similar value functions.}
The distributional Shapley value is naturally parameterized by the underlying data distribution $\D$.
For two distributions $\D_s$ and $\D_t$, given the value $\val(z;\pot,\D_s,m)$, what can we say about the value $\val(z;\pot,\D_t,m)$?
Intuitively, if $\D_s$ and $\D_t$ are similar under an appropriate metric, we'd expect that the values should not change too much.
Indeed, we can formally quantify how the distributional Shapley value is stable under distributional shift under the Wasserstein distance.

For two distributions $\D_s,\D_t$ over $\Z$,
let $\Gamma_{st}$ be the collection of joint distributions over $\Z \times \Z$, whose marginals are $\D_s$ and $\D_t$.\footnote{That is, for all $\gamma \in \Gamma_{st}$, if $(s,t) \sim \gamma$, then $s \sim \D_s$ and $t \sim \D_t$.}
Fixing a metric $d$ over $\Z$, the Wasserstein distance is the infimum over all such couplings $\gamma \in \Gamma_{st}$ of the expected distance between $(s,t) \sim \gamma$.
\begin{equation}
\label{eqnerstein}
W_1(\D_s,\D_t) \triangleq \inf_{\gamma \in \Gamma_{st}} ~\E_{(s,t)\sim \gamma} \left[d(s,t)\right]
\end{equation}
We formalize the idea that distributional Shapley values are stable under small perturbations to the underlying data distribution as follows.
\begin{theorem}
\label{thm:stable:dist}
Fix a metric space $(\Z,d)$ and let $\pot:\Z^* \to [0,1]$ be $\beta(k)$-Lipschitz stable with respect to $d$.  Suppose $\D_s$ and $\D_t$ are two distributions over $\Z$.
Then, for all $m \in \N$ and all $z \in \Z$,
\begin{gather*}
\card{\val(z;\pot,\D_s,m) - \val(z;\pot,\D_t,m)} \le \frac{2}{m} \sum_{k =1}^{m-1} k \beta(k) \cdot W_1(\D_s,\D_t).
\end{gather*}
\end{theorem}

\begin{proof}
For notational convenience, for any $z \in \Z$ and subset $S \subseteq \Z$, we denote $\Delta_z \pot(S) = \pot(S \cup \set{z}) - \pot(S)$.
Thus, fixing $z \in \Z$, we can write $\val(z;\pot,\D,m)$ as $\E_{k \sim [m]}\E_{S \sim \D^{k-1}}\left[\Delta_z \pot(S)\right]$.
We analyze $\E_{S \sim \D^{k-1}}\left[\Delta_z \pot(S)\right]$ for each fixed $k \in \set{2,\hdots,m}$ separately.\footnote{Note that for a fixed potential $U$, $m=1$ is uninteresting because both sides of the inequality are $0$; in particular, $\card{S}$ is always $0$, so the LHS is given by the difference $U(z) - U(z)$.}

Let $\gamma \in \Gamma_{st}$ be some coupling of $\D_s$ and $\D_t$.
Then, we can expand the expectation as follows.
\begin{align}
\E_{S \sim \D_s^{k-1}}\left[\Delta_z \pot(S)\right]
&=
\E_{S\times T \sim \gamma^{k-1}}\left[\Delta_z \pot(S)\right]\label{thm:eqn:wass:coup1}\\
&= \E_{S\times T
}\left[\Delta_z \pot(S) - \Delta_z \pot(T)\right] + \E_{S\times T
}\left[\Delta_z \pot(T) \right]\label{thm:eqn:wass:linearity}\\
&= \E_{S\times T
}\left[\Delta_z \pot(S) - \Delta_z \pot(T)\right] + \E_{T \sim \D_t^{k-1}}\left[\Delta_z \pot(T)\right]\label{thm:eqn:wass:coup2}
\end{align}
where (\ref{thm:eqn:wass:coup1}) and (\ref{thm:eqn:wass:coup2}) follow by the assumption that the marginals of $\gamma$ are $\D_s$ and $\D_t$; and (\ref{thm:eqn:wass:linearity}) follows by linearity of expectation.

To bound the first term of (\ref{thm:eqn:wass:coup2}), we expand the difference between $\Delta_z \pot(S)$ and $\Delta_z \pot(T)$ into a telescoping sum of $k$ pairs of terms, where we bound each pair to depend on a single draw $(s_i,t_i) \sim \gamma$.
For $S,T \in \Z^k$ and $i \in \set{0,\hdots,k}$, denote by $Z_i = \left(\bigcup_{j=i+1}^k s_j\right) \cup \left(\bigcup_{j=1}^i t_j\right)$; note that $Z_0 = S$ and $Z_k = T$.
Then, we can rewrite $\Delta_z \pot(S) - \Delta_z \pot(T)$ as follows.
\begin{align*}
\Delta_z \pot(S) - \Delta_z \pot(T)
&= \sum_{i=1}^k \Delta_z \pot(Z_{i-1}) - \Delta_z \pot(Z_i)
\end{align*}
Now suppose $\pot$ is $\beta(k)$-Lipschitz stable with respect to $d$; note that this implies $\Delta_z \pot$ is $2\beta(k)$-Lipschitz stable (because $\beta$ is non-increasing).
Then, we obtain the following bound.
\begin{align}
\E_{S\times T \sim \gamma^{k-1}}\left[\Delta_z \pot(S) - \Delta_z \pot(T)\right]
&=
\E_{S\times T \sim \gamma^{k-1}}\left[\sum_{i=1}^{k-1} \Delta_z \pot(Z_{i-1}) - \Delta_z \pot(Z_i)\right] \notag\\
&= \sum_{i=1}^{k-1} \E_{S,T \sim \gamma^{k-1}}\left[\Delta_z \pot(Z_{i-1}) - \Delta_z \pot(Z_i)\right]\notag \\
&=\sum_{i=1}^{k-1}
\E_{\substack{s_i,t_i \sim \gamma\\R \in \Z^{k-2}}}
\left[\Delta_z \pot(R \cup \set{s_i}) - \Delta_z \pot(R \cup \set{t_i})\right] \label{thm:eqn:wass:neighbor}\\
&\le 2\beta(k-1) \cdot \sum_{i=1}^{k-1}  \E_{(s_i,t_i) \sim \gamma}[d(s_i,t_i)]\label{thm:eqn:wass:stable}\\
&\le 2(k-1)\beta(k-1) \cdot \E_{(s,t) \sim \gamma}[d(s,t)]\label{thm:eqn:wass:iid}
\end{align}
where (\ref{thm:eqn:wass:neighbor}) notes $Z_{i-1}$ and $Z_i$ differ on only the $i$th data point; (\ref{thm:eqn:wass:stable}) follows from the assumption that $\Delta_z \pot$ is $2\beta(k)$-Lischitz stable and linearity of expectation; and finally (\ref{thm:eqn:wass:iid}) follows by the fact that each draw from $\gamma$ is i.i.d.

Finally, we note that the argument above worked for an arbitrary coupling in $\Gamma_{st}$; thus, we can express the difference in values in terms of the infimum over $\Gamma_{st}$.
\begin{align*}
&\val(z; \pot,\D_s,m) - \val(z; \pot,\D_t,m)\\
&\le \inf_{\gamma \in \Gamma_{st}}~\E_{k \sim [m]}\left[\E_{S\times T \sim \gamma^{k-1}}\left[\Delta_ \pot(S) - \Delta_z \pot(T)\right]\right]\\
&\le \frac{2}{m} \sum_{k = 2}^m (k-1)\beta(k-1) \inf_{\gamma \in \Gamma_{st}} \E_{(s,t) \sim \gamma}[d(s,t)]\\
&= \frac{2}{m} \sum_{k=1}^{m-1} k\beta(k) \cdot W_1(\D_s,\D_t)
\end{align*}
where the first summation is taken over $k \in \set{2,\hdots,m}$ as the term associated with $k=1$ is $0$.
\end{proof}

Note that the theorem bounds the difference in values under shifts in distribution holding the potential $\pot$ fixed.
Often in applications, we will take the potential function to depend on the underlying data distribution.
For instance, we may take to be a measure of population accuracy, e.g.\ $\pot_{\D_s} = 1-\E_{z \sim \D}\left[\ell_S(z)\right]$,
where $\ell_S(z)$ is the loss on a point $z \in \Z$ achieved by a model trained on the data set $S \subseteq \Z$.
In the case where we only have access to samples from $\D_s$, we still may want to guarantee that $\val(z;\pot_{\D_s},\D_s,m)$ and $\val(z;\pot_{\D_t},\D_t,m)$ are close.
Thankfully, such a result follows by showing that $\pot_{\D_s}$ is close to $\pot_{\D_t}$, and another application of the triangle inequality.
For instance, when the potential is based on the population loss for a Lipschitz loss function, we can bound the difference in the potentials, again, in terms of the Wasserstein distance.
\begin{align*}
\pot_{\D_t}(Z) - \pot_{\D_s}(Z) &=
\E_{s \sim \D_s}\left[\ell_Z(s)\right] - \E_{t \sim \D_t}\left[\ell_Z(t)\right]\\
&=\inf_{\gamma \in \Gamma_{st}}\E_{(s,t) \sim \gamma}\left[\ell_Z(s) - \ell_Z(t)\right]\\
&\le \inf_{\gamma \in \Gamma_{st}}\E_{s,t}\left[L \cdot d(s,t)\right]\\
&\le L \cdot W_1(\D_s,\D_t).
\end{align*}

\section{Efficiently Estimating Distributional Shapley Values}
\label{sec:alg}

Here, we describe an estimation procedure, \textsc{$\D$-Shapley}, for computing distributional Shapley values.  To begin, we assume that we can actually sample from the underlying $\D$.
Then, in Section~\ref{sec:alg:speedup}, we propose techniques to speed up the estimation and look into the practical issues of obtaining samples from the distribution.
The result of these considerations is a practically-motivated variant of the estimation procedure, \textsc{Fast-$\D$-Shapley}.
In Section~\ref{sec:alg:empirical}, we investigate how these optimizations perform empirically; we show that the strategies provide a way to smoothly trade-off the precision of the valuation for computational cost.

\subsection{Obtaining unbiased estimates}
The formulation from Theorem~\ref{thm:Dshapley} suggests a natural algorithm for estimating the distributional Shapley values of a set of points.
In particular, the distributional Shapley value $\val(z;\pot,\D,m)$ is the expectation of the marginal contribution of $z$ to $S \subseteq \Z$ on $\pot$, drawn from a specific distribution over data sets.
Thus, the change in performance when we add a point $z$ to a data set $S$ drawn from the correct distribution will be an unbiased estimate of the distributional Shapley value.
Consider the Algorithm~\ref{alg}, \textsc{$\D$-Shapley}, which given a subset $Z_0 \subseteq \Z$ of data, maintains for each $z \in Z_0$ a running average of $\pot(S \cup \set{z}) - \pot(S)$ over randomly drawn $S$.

\begin{algorithm}
\caption{\label{alg}~\textsc{$\D$-Shapley}}

\textbf{Fix:} \emph{potential $\pot:\Z^* \to [0,1]$; distribution $\D$; $m \in \N$}

\textbf{Given:} \emph{data set $Z \subseteq \Z$ to valuate; \# iterations $T \in \N$}

\begin{algorithmic}
\FOR{$z \in Z$}
\STATE $\val_1(z) \gets 0$\hfill\texttt{// initialize estimates}
\ENDFOR
\FOR{$t = 1,\hdots,T$}
\STATE Sample $S_t \sim \D^{k-1}$ for $k \sim [m]$
\FOR{$z \in Z$}
\STATE $\Delta_z\pot(S_t) \gets \pot(S_t \cup \set{z}) - \pot(S_t)$
\STATE $\val_{t+1}(z) \gets \frac{1}{t}\cdot \Delta_z\pot(S_t) + \frac{t-1}{t}\cdot \val_t(z)$
\\\hfill\texttt{// update unbiased estimate}
\ENDFOR
\ENDFOR
\STATE \textbf{return} $\set{(z,\val_T(z)) : z \in Z}$
\end{algorithmic}
\end{algorithm}

In each iteration, Algorithm~\ref{alg} uses a fixed sample $S_t$ to estimate the marginal contribution to $\pot(S_t \cup \set{z}) - \pot(S_t)$ for each $z \in Z$.
This reuse correlates the estimation errors between points in $Z$, but provides computational savings.
Recall that each evaluation of $\pot(S)$ requires training a ML model using the points in $S$; thus,
using the same $S$ for each $z \in Z$ reduces the number of models to be trained by $\card{Z}$ per iteration.
In cases where the $\pot(S \cup \set{z})$ can be derived efficiently from $\pot(S)$, the savings may be even more dramatic; for instance, given a machine-learned model trained on $S$, it may be significantly cheaper to derive a model trained on $S \cup \set{z}$ than retraining from scratch \cite{ginart2019making}.

The running time of Algorithm~\ref{alg} can naively be upper bounded by the product of the number of iterations before termination $T$, the cardinality $\card{Z}$ of the points to valuate, and the expected time to evaluate $\pot$ on data sets of size $k \sim [m]$.
We analyze the iteration complexity necessary to achieve $\eps$-approximations of $\val(z;\pot,\D,m)$ for each $z \in Z$.
\begin{theorem}
\label{thm:iterations:uniform}
Fixing a potential $U$ and distribution $\D$, and $Z \subseteq \Z$, suppose  $T \ge \Omega\left(\frac{\log(\card{Z}/\delta)}{\eps^2}\right)$.
Algorithm~\ref{alg} produces unbiased estimates and with probability at least $1-\delta$,
$\card{\val(z;\pot,\D,m) - \val_T(z)} \le \eps$.
for all $z \in Z$.
\end{theorem}
\begin{remark}
When understanding this (and future) formal approximation guarantees, it is important to note that we take $\eps$ to be an \emph{absolute} additive error.
Recall, however, that $\val(z;\pot,\D,m)$ is normalized by $m$; thus, as we take $m$ larger, the \emph{relative} error incurred by a fixed $\eps$ error grows.
In this sense, $\eps$ should typically scale inversely as $O(1/m)$.
\end{remark}

The claim follows by proving uniform convergence of the estimates for each $z \in Z$.
Importantly, while the samples in each iteration are correlated across $z,z' \in Z$,  fixing $z \in Z$, the samples $\Delta_z\pot(S_t)$ are independent across iterations.
We include a formal analysis in Appendix~\ref{app:alg}.

\subsection{Speeding up $\D$-Shapley: theoretical and practical considerations}
\label{sec:alg:speedup}

Next, we propose two principled ways to speed up the baseline estimation algorithm.
Under stability assumptions, the strategies maintain strong formal guarantees on the quality of the learned valuation.
We also develop some guiding theory addressing practical issues that arise from the need to sample from $\D$.
Somewhat counterintuitively, we argue that given only a fixed finite data set $B \sim \D^M$, we can still estimate values $\val(z;\pot,\D,m)$ to high accuracy, for $M$ that grows modestly with $m$.

\paragraph{Subsampling data and interpolation.}
Theorem~\ref{thm:stable:points} shows that for sufficiently stable potentials $\pot$, similar points have similar distributional Shapley values.
This property of distributional Shapley values is not only useful for inferring the values of points $z \in \Z$ that were not in our original data set, but also suggests an approach for speeding up the computations of values for a fixed $Z \subseteq \Z$.
In particular, to estimate the values for $z \in Z$ (with respect to a sufficiently Lipschitz-stable potential $\pot$) to $O(\eps)$-precision, it suffices to estimate the values for an $\eps$-cover of $Z$, and interpolate (e.g.\ via nearest neighbor search).
Standard arguments show that random sampling is an effective way to construct an $\eps$-cover \cite{har2011geometric}.

As our first optimization, in Algorithm~\ref{alg:fast}, we reduce the number of points to valuate through subsampling.
Given a data set $Z$ to valuate, we first choose a random subset $Z_p \subseteq Z$ (where each $z \in Z$ is subsampled into $Z_p$ i.i.d.\ with some probability $p$); then, we run our estimation procedure on the points in $Z_p$; finally, we train a regression model on $(z,\val_T(z))$ pairs from $Z_p$ to predict the values of the points from $Z \setminus Z_p$.
By varying the choice of $p \in [0,1]$, we can trade-off running time for quality of estimation: $p\approx 1$ recovers the original \textsc{$\D$-Shapley} scheme, whereas $p\approx 0$ will be very fast but likely produce noisy valuations.

\paragraph{Importance sampling for smaller data sets.}
To understand the running time of Algorithm~\ref{alg} further, we denote the time to evaluate $\pot$ on a set of cardinality $k \in \N$ by $R(k)$.\footnote{We assume that the running time to evaluate $\pot(S)$ is a function of the cardinality of $S$ (and not other auxiliary parameters).}
As such, we can express the asymptotic expected running time as $\card{Z} \cdot T \cdot \E_{k \sim [m]}\left[R(k)\right]$.
Note that when $\pot(S)$ corresponds to the accuracy of a model trained on $S$, the complexity of evaluating $\pot(S)$ may grow significantly with $\card{S}$.
At the same time,
as the data set size $k$ grows, the marginal effect of adding $z \in Z$ to the training set tends to decrease; thus, we should need fewer large samples to accurately estimate the marginal effects.
Taken together, intuitively,
biasing the sampling of $k \in [m]$ towards smaller training sets could result in a faster estimation procedure with similar approximation guarantees.

Concretely, rather than sampling $k \sim [m]$ uniformly, we can importance sample each $k$ proportional to some non-uniform weights $\set{w_k : k \in [m]}$, where the weights decrease for larger $k$.
More formally, we weight the draw of $k$ based on the stability of $\pot$.
Algorithm~\ref{alg:fast} takes as input a set of importance weights $w = \set{w_k}$ and samples $k$ proportionally; without loss of generality, we assume $\sum_{k} w_k = 1$ and let $k \sim [m]_w$ denote a sample drawn such that $\Pr[k] = w_k$.
We show that for the right choice of weights $w$, sampling $k \sim [m]_w$ improves the overall running time, while maintaining $\eps$-accurate unbiased estimates of the values $\val(z;\pot,\D,m)$.

\begin{theorem}[Informal]
\label{thm:biased}
Suppose $\pot$ is $O(1/k)$-deletion stable and can be evaluated on sets of cardinality $k$ in time $R(k) \ge \Omega(k)$.
For $p \in [0,1]$ and $w = \set{w_k \propto 1/k}$,
Algorithm~\ref{alg:fast} produces estimates that with probability $1-\delta$, are $\eps$-accurate for all $z \in Z_p$ and
runs in expected time
\begin{align*}
RT_w(m) &\le \tilde{O}\left(p\cdot \card{Z} \cdot \frac{\log(\card{Z}/\delta)\cdot R(m)}{\eps^2m^2}\right).
\end{align*}
\end{theorem}

To interpret this result, note that if the subsampling probability $p$ is large enough that $Z_p$ will $\eps$-cover $Z$, then using a nearest-neighbor predictor as $\mathcal{R}$ will produce $O(\eps)$-estimates for all $z \in Z$.
Further, if we imagine $\eps = \Theta(1/k)$, then the computational cost grows as the time it takes to train a model on $m$ points scaled by a factor logarithmic in $\card{Z}$ and the failure probability.
In fact, Theorem~\ref{thm:biased} is a special case of a more general theorem that provides a recipe for devising an appropriate sampling scheme based on the stability of the potential $\pot$.
In particular, the general theorem (stated and proved in Appendix~\ref{app:alg}) shows that the more stable the potential, the more we can bias sampling in favor of smaller sample sizes.

\begin{algorithm}[t!]
\caption{\label{alg:fast}~\textsc{Fast-$\D$-Shapley}}

\textbf{Fix:} \emph{potential $\pot:\Z^* \to [0,1]$; distribution $\D$; $m \in \N$}

\textbf{Given:} \emph{valuation set $Z \subseteq \Z$; database $B \sim \D^M$; \#~iterations $T \in \N$;\\subsampling rate $p \in [0,1]$; importance weights~$\set{w_k}$; regression algorithm $\mathcal{R}$}

\begin{algorithmic}
\STATE Subsample $Z_p \subseteq Z$ s.t.\ $z \in Z_p$ w.p.\ $p$ for all $z \in Z$
\FOR{$z \in Z_p$}
\STATE $\val_1(z) \gets 0$\hfill\texttt{// initialize estimates}
\ENDFOR
\FOR{$t = 1,\hdots,T$}
\STATE Sample $S_t \sim B^{k-1}$ for $k \sim [m]_w$
\FOR{$z \in Z_p$}
\STATE $\Delta_z\pot(S_t) \gets \pot(S_t \cup \set{z}) - \pot(S_t)$
\STATE $\val_{t+1}(z) \gets \frac{1}{t}\cdot \frac{\Delta_z\pot(S_t)}{w_k m} + \frac{t-1}{t}\cdot \val_t(z)$
\\\hfill\texttt{// update unbiased estimate}
\ENDFOR
\ENDFOR
\STATE $h \gets \mathcal{R}\left(\set{(z,\val_T(z)) : z \in Z_p}\right)$
\\\hfill\texttt{// regress on (z,val(z)) pairs}
\STATE \textbf{return} $\set{(z,h(z)) : z \in Z}$
\end{algorithmic}
\end{algorithm}

\paragraph{Estimating distributional Shapley from data.}
Estimating distributional Shapley values $\val(z;\pot,\D,m)$ requires samples from the distribution $\D$.
In practice, we often want evaluate the values with respect to a distribution $\D$ for which we only have some database $B \sim \D^M$ for some large (but finite) $M \in \N$.
In such a setting, we need to be careful; indeed, avoiding artifacts from a single draw of data is the principle motivation for introducing the distributional Shapley framework.
In fact, the analysis of Theorem~\ref{thm:biased} also reveals an upper bound on how big the database should be in order to obtain accurate estimates with respect to $\D$.
As a concrete bound, if $\pot$ is $O(1/k)$-deletion stable and we take $\eps = \Theta(1/m)$ error, then the database need only be
\begin{equation*}
M \le \tilde{O}\left(m \cdot \log(\card{Z}/\delta)\right).
\end{equation*}
In other words, for a sufficiently stable potential $\pot$,
the data complexity grows modestly with $m$.
Note that, again, this bound leverages the fact that in every iteration, we reuse the same sample $S_t \sim \D^k$ for each $z \in Z$.
See Appendix~\ref{app:alg} for a more detailed analysis.

In practice, we find that sampling subsets of data from the database with replacement works well; we describe the full procedure in Algorithm~\ref{alg:fast}, where we denote an i.i.d.\ sample of $k$ points drawn uniformly from the database as $S \sim B^k$.
Finally, we note that ideally, $m$ should be close to the size of the training sets that model developers to use; in practice, these data set sizes may vary widely.
One appealing aspect of both \textsc{$\D$-Shapley} algorithms is that when we estimate values with respect to $m$, the samples we obtain also allow us to simultaneously estimate $\val(z;\pot,\D,m')$ for any $m' \le m$.
Indeed, we can simply truncate our estimates to only include samples corresponding to $S_t$ with $\card{S_t} \le m'$.

\subsection{Empirical performance}
\label{sec:alg:empirical}

We investigate the empirical effectiveness of the distributional Shapley framework by running experiments in three settings on large real-world data sets.
The first setting uses the UK Biobank data set, containing the genotypic and phenotypic data of individuals in the UK~\cite{sudlow2015ukb}; we evaluate a task of predicting whether the patient will be diagnosed with breast cancer using 120 features.
Overall, our data has 10K patients (5K diagnosed positively); we use 9K patients as our database ($B$), and take classification accuracy on a hold-out set of 500 patients as the performance metric ($\pot$).
The second data set is Adult Income where the task is to predict whether income exceeds $\$50$K/yr given 14 personal features~\cite{Dua:2019}.
With 50K individuals total, we use 40K as our database, and classification accuracy on 5K individuals as our performance metric.
In these two experiments, we take the maximum data set size $m = 1$K and $m=5$K, respectively.

For both settings, we first run \textsc{$\D$-Shapley} without optimizations as a baseline.
As a point of comparison, in these settings the computational cost of this baseline is on the same order as running the \textsc{TMC-Shapley} algorithm of \cite{datashapley} that computes the data Shapley values $\sh(z;\pot,B)$ for each $z$ in the data set $B$.
Given this baseline,
we evaluate the effectiveness of the proposed optimizations, using weighted sampling and interpolation (separately), for various levels of computational savings.
In particular, we vary the sampling weights $\set{w_k}$ and subsampling probability $p$ to vary the computational cost (where weighting towards smaller $k$ and taking $p$ smaller each yield more computational savings).
All algorithms are truncated when the average absolute change in value in the past $100$ iterations is less than $1\%$.

\begin{figure}[b!]
\centering
\includegraphics[width=\linewidth]{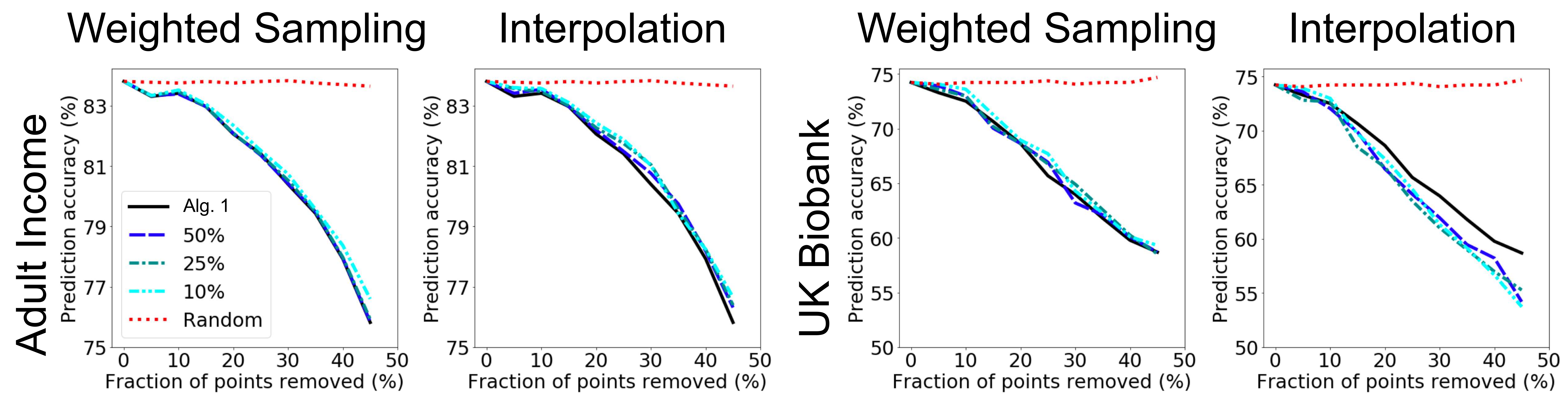} 
\caption{\textbf{Point removal performance.}
Given a data set and task, we iteratively a point, retrain the model, and evaluate its performance.
Each curve corresponds to a different point removal order, based on the estimated distributional Shapley values (compared to random).
For example, the $10\%$ curve correspond to estimating values with $10\%$ of the baseline computation of Algorithm~\ref{alg}.
We plot classification accuracy vs.\ fraction of data points removed from the training set, for each task and each optimization method.
\label{fig:logistic}}
\end{figure}

To evaluate the quality of the distributional Shapley estimates, we perform a point removal experiment, as proposed by \cite{datashapley}, where given a training set, we iteratively remove points, retrain the model, and observe how the performance changes.
In particular, we remove points from most to least valuable (according to our estimates), and compare to the baseline of removing random points.
Intuitively, removing high value data points should result in a more significant drop in the model's performance.
We report the results of this point removal experiment using the values determined using the baseline Algorithm~\ref{alg}, as well as various factor speed-ups (where $t\%$ refers to the computational cost compared to baseline).

\begin{figure}[t!]
\centering
\includegraphics[width=\linewidth]{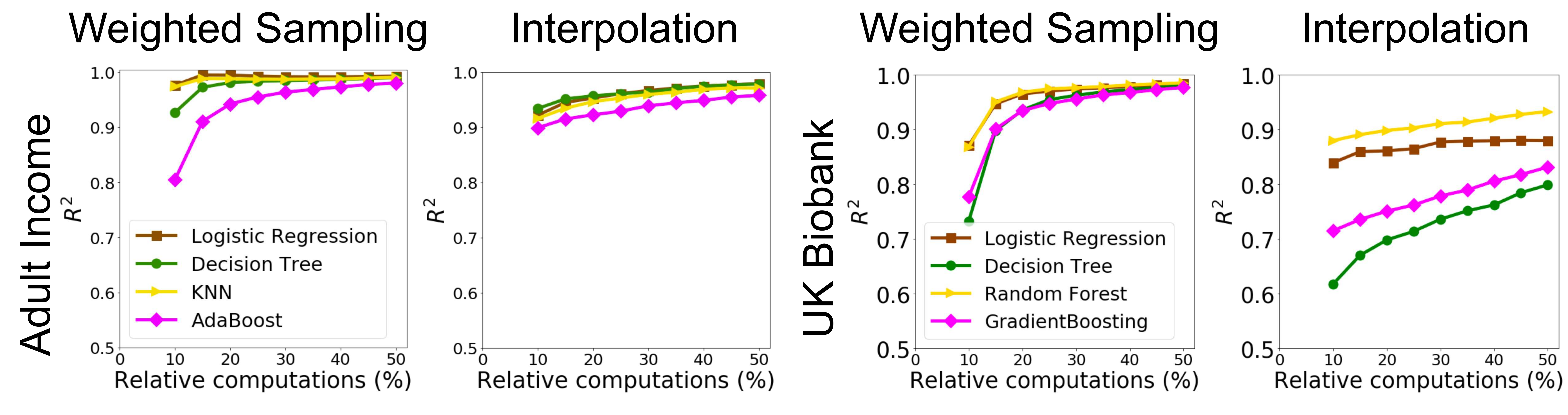} 
\caption{\textbf{Smooth trade-off between computation and recovery.}
For each task, we plot the $R^2$ coefficient between the values computed using Algorithm~\ref{alg} vs.\ the relative computational cost (as in Figure~\ref{fig:logistic}).
The results show that there is a smooth trade-off between the recovery precision of the distributional Shapley values and the cost, across a wide range of learning algorithms.
\label{fig:speedup}}
\end{figure}

As Figure~\ref{fig:logistic} demonstrates, when training a logistic regression model, removing the high distributional Shapley valued points causes a sharp decrease in accuracy on both tasks, even when using the most aggressive weighted sampling and interpolation optimizations.
Appendix~\ref{app:exp} reports the results for various other models.
As a finer point of investigation, we report the correlation between the estimated values without optimizations and with various levels of computational savings, for a handful of prediction models.
Figure~\ref{fig:speedup} plots the $R^2$ curves and shows that the optimizations provide a smooth interpolation between computational cost and recovery, across every model type.
It is especially interesting that these trade-offs are consistently smooth across a variety of models using the $01$-loss, which do not necessarily induce a potential $\pot$ with formal guarantees of stability.

In our final setting, we push the limits of what types of data can be valuated.
Specifically, by combining both weighted sampling and interpolation (resulting in a $500\times$ speed-up), we estimate the values of $50$K images from the CIFAR10 data set; valuating this data set would be prohibitively expensive using prior Shapley-based techniques.
In particular, to obtain accurate estimates for each point, \textsc{TMC-Shapley} would require an unreasonably large number of Monte Carlo iterations due to the sheer size of the data base to valuate.
We valuate points based on an image classification task,
and demonstrate that the estimates identify highly valuable points, in Appendix~\ref{app:exp}.

\section{Case Study: Consistently Pricing Data}
\label{sec:case}

Next, we consider a natural setting where a data broker wishes to sell data to various buyers. Each buyer could already own some private data.
In particular, suppose the broker plans to sell the set $S$ and a buyer holds a private data set $B$;
in this case, the relevant values are the data Shapley values $\sh(z;\pot,B \cup S)$ for each $z \in S$.
Within the original data Shapley framework, computing these values requires a single party to hold both $B$ and $S$.
For a multitude of financial and legal concerns, neither party may be willing to send their data to the other before agreeing to the purchase.
Such a scenario represents a fundamental limitation of the non-distributional Shapley framework that seemed to jeopardize its practical viability.
We argue that the distributional Shapley framework largely resolves this particular issue:
without exchanging data up front, the broker simply estimates the values $\val(z;\pot,\D,m)$; in expectation, these values will accurately reflect the value to a buyer with a private data set $B$ drawn from a distribution close to $\D$.

\begin{figure}[t!]
\centering
\includegraphics[width=\linewidth]{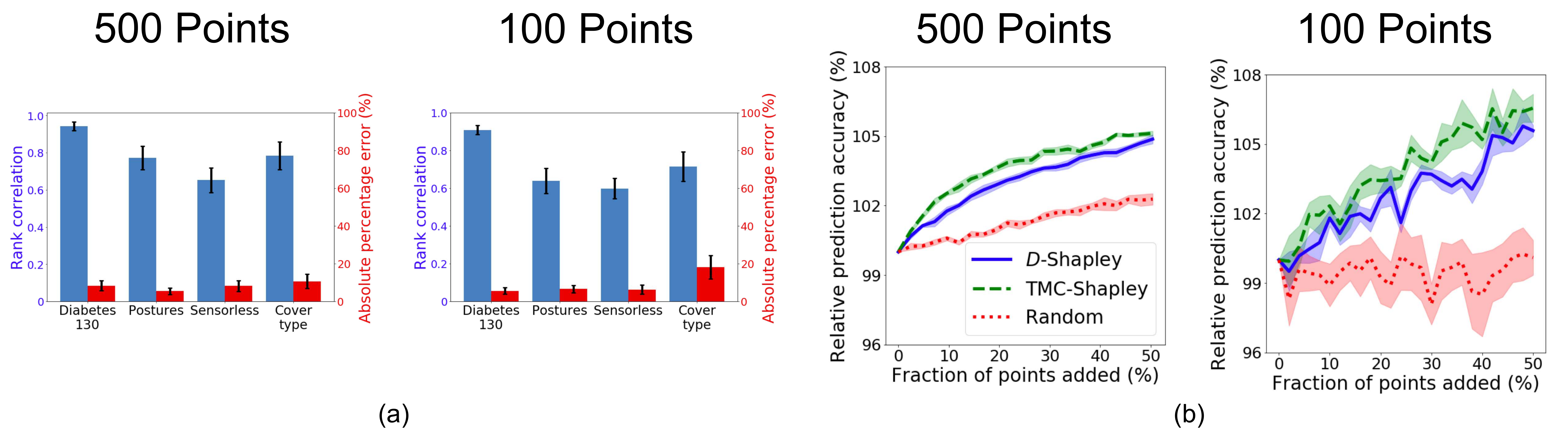} 
\caption{\textbf{Consistent Pricing.}
Each buyer holds a data set $B$; the seller sells a data set $S$, where $\card{B} = \card{S} = m$.
We compare the values estimated by the seller $\val(z;\pot,\D,m)$ and $\sh(z;\pot,B \cup S)$.
(a)  For various data sets and two data set sizes ($m = 100$ and $m = 500$): in blue, we plot the average rank correlation between $\val(z)$ and $\sh(z)$ for $z \in S$; in red, we plot
the average absolute percentage error between the seller's and buyer's estimates.
(b) Points from $S$ are added to $B$ in three different orders: according to $\val$ ($\D$-Shapley), according to $\sh$ (TMC), and randomly.
The plot shows the change in the accuracy of the model, relative to its performance using the buyer's initial dataset, as the points are added; shading indicates standard error of the mean.
\label{fig:selling}}
\end{figure}

We report the results of this case study on four large different data sets in Figure~\ref{fig:selling}, whose details are included in Appendix~\ref{app:case}.
For each data set, a set of buyers holds a small data set $B$ ($100$ or $500$ points), and the broker sells them a data set $S$ of the same size; the buyers then valuate the points in $S$ by running the \textsc{TMC-Shapley} algorithm of \cite{datashapley} on $B \cup S$.
In Figure~\ref{fig:selling}(a), we show that the rank correlation between the broker's distributional estimates $\val(z;\pot,\D,m)$ and the buyer's observed values $\sh(z;\pot,B \cup S)$ is generally high.
Even when the rank correlation is a bit lower ($\approx 0.6$), the broker and buyer agree on the value of the set as a whole.
Specifically, we observe that the seller's estimates are approximately unbiased,
and the absolute percentage error is low, where\\
\begin{equation*}
APE =  \dfrac{\card{\sum_{z \in S} \val(z;\pot,\D,m) - \sh(z;\pot,B \cup S)}}{\sum_{z \in S} \val(z;\pot,\D,m)}.
\end{equation*}
In Figure~\ref{fig:selling}(b), we show the results of a point addition experiment for the Diabetes130 data set.
Here, we consider the effect of adding the points of $S$ to $B$ under three different orderings: according to the broker's estimates $\val(z;\pot,\D,m)$, according to the buyer's estimates $\sh(z;\pot,B \cup S)$, and under a random ordering.
We observe that the performance (classification accuracy) increase by adding the points according to $\val(z)$ and according to $\sh(z)$ track one another well; after the addition of all of $S$, the resulting models achieve essentially the same performance and considerably outperforming random.
We report results for the other data sets in Appendix~\ref{app:case}.

\section{Discussion}
\label{sec:discuss}

The present work makes significant progress on understanding statistical aspects in determining the value of data.
In particular, by reformulating the data Shapley value as a distributional quantity, we obtain a valuation function that does not depend on a fixed data set; reducing the dependence on the specific draw of data eliminates inconsistencies in valuation that can arise to sampling artifacts.
Further, we demonstrate that the distributional Shapley framework provides an avenue to valuate data across a wide variety of tasks, providing stronger theoretical guarantees and orders of magnitude speed-ups over prior estimation schemes.
In particular, the stability results that we prove for distributional Shapley (Theorems~\ref{thm:stable:points} and \ref{thm:stable:dist}) are not generally true for the original data Shapley due to its dependence on a fixed dataset.

One outstanding limitation of the present work is the reliance on a known task, algorithm, and performance metric (i.e.\ taking the potential $\pot$ to be fixed).
We propose reducing the dependence on these assumptions as a direction for future investigations; indeed, very recent work has started to chip away at the assumption that the learning algorithm is fixed in advance \cite{yona2019s}.

The distributional Shapley perspective also raises
the thought-provoking research question of
whether we can valuate data while protecting the privacy of individuals who contribute their data.
One severe limitation of the data Shapley framework, is that the value of every point depends nontrivially on every other point in the data set.
In a sense, this makes the data Shapley value an inherently non-private value:  the estimate of $\sh(z;\pot,B)$ for a point $z \in B$ reveals information about the other points in $B$.
By marginalizing the dependence on the data set, the distributional Shapley framework opens the door for to estimating data valuations while satisfying strong notions of privacy, such as differential privacy \cite{dworkDP}.
Such an estimation scheme could serve as a powerful tool amidst increasing calls to ensure the privacy of and compensate individuals for their personal data \cite{datacoops}.

\clearpage
\bibliographystyle{alpha}
\bibliography{shapley}

\setcounter{figure}{0}    
\renewcommand{\figurename}{Supplementary Figure}

\clearpage
\appendix

\section{Review of Shapley Axioms}
\label{app:axioms}

Here, we provide a high-level review of the axioms that Shapley used to describe an equitable valuation function  \cite{shapley1953value}.
We consider the data Shapley setting, letting $\sh(z;U,B)$ denote the value of $z \in B$ for a finite subset $B \subseteq \Z$ with respect to potential $\pot:\Z^* \to [0,1]$.
\begin{itemize}
    \item \emph{Symmetry} -- Consider $z_i,z_j \in B$; suppose for all $S \subseteq B\setminus \set{z_i,z_j}$, $\pot(S \cup \set{z_i}) = \pot(S \cup \set{z_j})$.
    Then, $$\sh(z_i;U,B) = \sh(z_j;U,B).$$
    That is, if two data points are equivalent, then they should receive the same value.
    \item \emph{Null player} -- Consider $z \in B$; suppose for all $S \subseteq B \setminus \set{z}$, $\pot(S \cup \set{z}) = \pot(S)$.  Then, $$\sh(z;U,B) = 0.$$
    That is, if a data point contributes no marginal gain in potential to any nontrivial subset, then it receives no value.
    \item \emph{Additivity} -- Consider two potentials $\pot_1,\pot_2$. For all $z \in B$, $$\sh(z;\pot_1+\pot_2,B) = \sh(z;\pot_1,B) + \sh(z;\pot_2,B).$$
    That is, the value of a data point with respect to the combination of two tasks (addition of two potentials) is the sum of the values with respect to each task (potential) separately.
\end{itemize}
\begin{theorem}[\cite{shapley1953value}~]
The Shapley value is the unique valuation function that satisfies the symmetry, null player, and additivity axioms.
\end{theorem}
Additionally, the Shapley value satisfies the desirable property that it allocates all of the value to the contributors.
\begin{itemize}
    \item \emph{Efficiency} -- The sum of the individuals' Shapley values equals the value of the coalition.
$$ \sum_{z \in B} \sh(z;U,B) = \pot(B) - \pot(\emptyset).$$
\end{itemize}

It is straightforward to verify that the distributional Shapley value immediately inherits the properties of symmetry, null player, and additivity (by linearity of expectation).
Further, it satisfies an on-average variant of efficiency.
\begin{proposition}
Given a potential $\pot$ and a data distribution $\D$, for $m \in \N$,
\begin{equation*}
\E_{z \sim \D}\left[\val(z;\pot,\D,m)\right] = \frac{\E_{B \sim \D^m}\left[\pot(B)\right] - \pot(\emptyset)}{m}.
\end{equation*}
\end{proposition}
\begin{proof}
We expand the expected distributional Shapley value with its definition and then apply linearity of expectation.
\begin{align*}
\E_{z \sim \D}\left[\val(z;\pot,\D,m)\right]
&= \E_{z \sim \D}\left[\E_{\substack{k\sim [m]\\S \sim \D^{k-1}}}\left[\pot(S \cup \set{z}) - \pot(S)\right]\right]\\
&= \frac{1}{m} \cdot \sum_{k=1}^m \left(\E_{\substack{z \sim \D\\S \sim \D^{k-1}}}\left[\pot(S \cup \set{z})\right] - \E_{S \sim \D^{k-1}}\left[\pot(S)\right]\right)\\
&= \frac{1}{m} \cdot \sum_{k=1}^m \left(\E_{S_k \sim \D^{k}}\left[\pot(S_k)\right] - \E_{S_{k-1} \sim \D^{k-1}}\left[\pot(S_{k-1})\right]\right)\\
&= \frac{1}{m} \cdot \left(\E_{B \sim \D^m}\left[\pot(B)\right] - \pot(\emptyset)\right)
\end{align*}
\end{proof}

\section{Distributional Shapley Value for Mean Estimation}
\label{app:meanest}

\begin{proposition*}[Restatement of Proposition~2.4]
Suppose $\D$ has bounded second moments.
Then for $z \in \Z$ and $m \in \N$, $\val(z;\pot_\mu,\D,m)$ for mean estimation over $\D$ is given by
\begin{gather*}
\frac{\E_{S \sim \D^m}\left[\pot(S)\right]}{m} + 
\frac{C_m}{m}\cdot \left(\E_{s \sim \D}\left[\norm{s-\mu}^2\right] - \norm{z - \mu}^2\right)
\end{gather*}
for an explicit constant $C_m = \Theta(1)$ determined by $m$.
\end{proposition*}

\begin{proof}
Consider the unsupervised learning task of mean estimation using the empirical estimator.
Specifically, suppose we receive samples from some distribution $\D$ supported on $\R^d$ with mean $\mu = \E_{s \sim \D}[s]$ and bounded second moments.
Given a subset $S \subseteq \R^d$, we consider the empirical estimator $\hat{\mu}_S = \frac{1}{\card{S}}\cdot \sum_{s \in S} s$.
We define a potential $\pot(S)$ by the performance of the empirical estimator.
For notational convenience, let $\E_{s \sim \D}\left[\norm{s - \mu}^2\right] = R^2$ for some $R = \Theta(1)$.
\begin{align*}
\pot(S) &= \E_{s \sim \D}\left[\norm{s - \mu}^2\right] - \norm{\hat{\mu}_S - \mu}^2\\
&= R^2 - \norm{\hat{\mu}_S - \mu}^2
\end{align*}
By convention, we will assume that $\pot(\emptyset) = 0$.
As such, we can evaluate the difference in potentials as follows.
\begin{align*}
&=\left(R^2 - \norm{\mu - \hat{\mu}_{S\cup \set{z}}}^2\right) - \left(R^2 -  \norm{\mu - \hat{\mu}_S}^2\right)\\
&= \norm{\mu - \hat{\mu}_S}^2 - \norm{\mu - \hat{\mu}_{S\cup \set{z}}}^2\\
\end{align*}
Importantly, note that we can relate $\hat{\mu}_{S \cup \set{z}}$ to $\hat{\mu}_S$.
\begin{equation*}
\hat{\mu}_{S \cup \set{z}} = \hat{\mu}_S + \frac{1}{k}\cdot \left(z - \hat{\mu}_S\right)
\end{equation*}
Using these expressions, we can expand the distributional Shapley value into a form that will be convenient to work with.
\begin{align*}
\val(z;\pot,\D,m) &=
\E_{\substack{k \sim [m]\\S \sim \D^{k-1}}}\left[\pot(S \cup \set{z}) - \pot(S)\right]\\
&= \frac{1}{m}\cdot \sum_{k=1}^m \E_{S \sim \D^{k-1}}\left[\pot(S \cup \set{z}) - \pot(S)\right]\\
&= \frac{1}{m}\cdot \left(\pot(\set{z}) - \pot(\emptyset) + \sum_{k=2}^m \E_{S \sim \D^{k-1}}\left[\pot(S \cup \set{z}) - \pot(S)\right]\right)\\
&= \frac{1}{m}\cdot \left(R^2 - \norm{z-\mu}^2 + \sum_{k=2}^m \E_{S \sim \D^{k-1}}\left[ \norm{\mu - \hat{\mu}_S}^2 - \norm{\mu - \hat{\mu}_{S\cup \set{z}}}^2\right]\right)
\end{align*}
We, thus, focus our efforts on bounding the summation from $k=2$ to $m$.
As such, we can evaluate the difference in potentials within the expectation as follows.
\begin{align*}
&\phantom{=}\norm{\mu - \hat{\mu}_S}^2 - \norm{\mu - \hat{\mu}_{S\cup \set{z}}}^2\\
&=\norm{\mu - \hat{\mu}_S}^2 - \norm{\mu - \hat{\mu}_S - \frac{1}{k}\cdot \left(z-\hat{\mu}_S\right)}^2\\
&= \norm{\mu - \hat{\mu}_S}^2 - \left(\norm{\mu - \hat{\mu}_S}^2 + \frac{1}{k^2} \cdot \norm{z - \hat{\mu}_S}^2 - \frac{2}{k}\cdot \left\langle \mu - \hat{\mu}_S, z - \hat{\mu}_S \right\rangle\right)\\
&=\frac{2}{k}\cdot \left\langle \mu - \hat{\mu}_S, z - \hat{\mu}_S \right\rangle - \frac{1}{k^2} \cdot \norm{z - \hat{\mu}_S}^2
\end{align*}
Taking an expectation over $S \sim \D^{k-1}$, we can simplify each term in the summation separately;
first, some identities that will be useful and hold for all $n \in \N$:
\begin{gather}
\E_{S \sim \D^{n}}\left[\hat{\mu}_S\right] = \mu \label{expectation}\\
\E_{S \sim \D^{n}}\left[\langle \mu - \hat{\mu}_S, q \rangle \right] = 0\label{q}\\
\E_{S \sim \D^{n}}\left[\norm{\hat{\mu}_S}^2 - \norm{\mu}^2\right]
= \E_{S \sim \D^{n}}\left[\norm{\hat{\mu}_S - \mu}^2\right]
= \frac{1}{n}\cdot \E_{s\sim \D}\left[\norm{s-\mu}^2\right]\label{variance}
\end{gather}
where (\ref{expectation}) follows because $\hat{\mu}_S$ an unbiased estimator of $\mu$; (\ref{q}) holds for all $q \in \R^d$; and (\ref{variance}) is a well-known fact that can be derived using (\ref{expectation}) and (\ref{q}).

Beginning with the first inner product.
\begin{align*}
\E_{S \sim \D^{k-1}}\left[\langle \mu - \hat{\mu}_S, z - \hat{\mu}_S \rangle\right]
&= \E_{S \sim \D^{k-1}}\left[\langle \mu - \hat{\mu}_S, \mu - \hat{\mu}_S \rangle\right]\addtag\label{applying:q1}\\
&= \E_{S \sim \D^{k-1}}\left[\norm{\mu - \hat{\mu}_S}^2\right]\\
&= \frac{1}{k-1}\cdot \E_{s \sim \D}\left[\norm{s - \mu}^2\right]\addtag\label{applying:variance1}
\end{align*}
where (\ref{applying:q1}) applies (\ref{q}) with $q = \mu - z$ and (\ref{applying:variance1}) applies (\ref{variance}).

Expanding the next term.
\begin{align*}
\E_{S \sim \D^{k-1}}\left[\norm{z-\hat{\mu}_S}^2\right]
&= \E_{S \sim \D^{k-1}}\left[\norm{z}^2 + \norm{\hat{\mu}_S}^2 - 2 \langle z, \hat{\mu}_S\rangle + \norm{\mu}^2 - \norm{\mu}^2\right]\\
&= \E_{S \sim \D^{k-1}}\left[\norm{z}^2 - 2 \langle z, \hat{\mu}_S\rangle + \norm{\mu}^2\right]
+ \E_{S \sim \D^{k-1}}\left[\norm{\hat{\mu}_S}^2 - \norm{\mu}^2\right]\\
&= \norm{z - \mu}^2 + \frac{1}{k-1}\cdot \E_{s \sim \D}\left[\norm{s - \mu}^2\right]\addtag\label{linearity-zero}
\end{align*}
where (\ref{linearity-zero}) follows by
applying linearity of expectation and (\ref{expectation}) to the first term and (\ref{variance}) to the second term.

Thus, in all, the value can be expressed as follows. \begin{align*}
&\sum_{k = 2}^m\E_{S \sim \D^{k-1}}\left[\frac{2}{k}\cdot \left\langle \mu - \hat{\mu}_S, z - \hat{\mu}_S \right\rangle - \frac{1}{k^2} \cdot \norm{z - \hat{\mu}_S}^2\right]\\
&= \sum_{k=2}^m \left(\frac{2}{k\cdot(k-1)}\cdot \E_{s \sim \D}\left[\norm{s - \mu}^2\right] - \frac{1}{k^2\cdot(k-1)}\cdot\E_{s \sim \D}\left[\norm{s - \mu}^2\right] - \frac{1}{k^2}\cdot \norm{z - \mu}^2\right)\\
&= \sum_{k=2}^m \left(\frac{2R^2 - \norm{z - \mu}^2}{k\cdot(k-1)} - \frac{R^2 - \norm{z-\mu}^2}{k^2\cdot(k-1)}\right)\\
&= \frac{m-1}{m} \cdot \left(2R^2 - \norm{z - \mu}^2\right) + c(m)\cdot \left(R^2 - \norm{z - \mu}^2\right)\\
&= \frac{m-1}{m} \cdot R^2 + \left(1- 1/m + c(m)\right)\cdot \left(R^2 - \norm{z - \mu}^2\right)
\end{align*}
where $\sum_{k=2}^m \frac{1}{k\cdot(k-1)} = \frac{m-1}{m}$ and we take $c(m) = \sum_{k=2}^m \frac{1}{k^2\cdot(k-1)}$.
Thus, plugging this expression back into our original expansion.
\begin{align*}
&\val(z;\pot,\D,m)\\
&= \frac{1}{m}\cdot \left(\frac{m-1}{m}\cdot \left(2R^2 - \norm{z-\mu}^2\right) + (1+c(m))\cdot\left(R^2 - \norm{z-\mu}^2\right)\right)\\
&= \frac{m-1}{m^2}\cdot R^2 + \frac{C(m)}{m} \cdot \left(R^2 - \norm{z-\mu}^2\right)\\
&= \frac{1}{m}\cdot \left(C(m) \cdot \left(\E_{s \sim \D}\left[\norm{s-\mu}^2\right] - \norm{z - \mu}^2\right) + \left(\E_{s \sim \D}\left[\norm{s-\mu}^2\right] - \E_{S \sim \D^m}\left[\norm{\hat{\mu}_S - \mu}^2\right]\right)\right)
\end{align*}
where $C(m) = 2 - 1/m + c(m) = \Theta(1)$ is an explicit function of $m$, and we use the fact that $\E_{S \sim \D^m}\left[\norm{\hat{\mu}_S - \mu}^2\right] = \frac{1}{m} \cdot R^2$. \end{proof}

\section{Lipschitz Stability of RKHS}
\label{app:rkhs}

Suppose $\Z = \X \times \Y$;
let $\F$ to denote a Reproducing Kernel Hilbert Space (RKHS), with associated feature map $\varphi:\X \to \F$, inner product $\langle \cdot, \cdot \rangle_\F$, and norm $\norm{\cdot}_\F$,
such that for all $f \in \F$,
\begin{equation*}
f(x) = \langle f, \varphi(x) \rangle_\F
\end{equation*}
Given $\F$, we define a natural metric over $\Z$, where for labeled pairs $z_i = (x_i,y_i)$ and $z_j = (x_j,y_j)$,
\begin{equation*}
d_\F(z_i,z_j) = \begin{cases}
\norm{\varphi(x_i) - \varphi(x_j)}_\F & \text{ if}~ y_i = y_j\\
+\infty & \text{ o.w.}
\end{cases}
\end{equation*}
That is, the distance is given by the RKHS norm if $x_i$ and $x_j$ have the same label, and are arbitrarily dissimilar otherwise.

We define the potential of a subset $S \subseteq \Z$ for an RKHS learning problem as the performance achieved when training using $S$.
Specifically, suppose $\ell:[0,1]\times[0,1] \to \R^+$ is an $L$-Lipschitz, convex loss function and $\D$ is a distribution supported on $\Z$.
We define the potential function $\pot_\F:\Z^* \to [0,1]$ in terms of the population loss over $\D$ of the following regularized ERM.
\begin{gather*}
\pot_\F(S) = 1 - \E_{(x,y) \sim \D}\left[\ell(f_S(x),y)\right]\\
\text{where }~f_S = \argmin_{f \in \F} \set{\er_S(f) + \frac{\lambda}{2}\norm{f}^2_\F}
\end{gather*}
where $\er_S(f) = \frac{1}{\card{S}} \sum_{(x,y) \in S} \ell(f(x),y)$ and $\lambda > 0$.

\begin{lemma}[Learning RKHS is Lipschitz stable]
Suppose $\F$ is a RKHS with feature map $\phi:\X \to \F$.
Let $\D$ be a distribution over $\Z = \X \times \Y$ such that $\E_{(x,y)\sim \D}\left[\norm{\phi(x)}_\F\right] = R$.
Then, $\pot_\F:\Z^* \to [0,1]$ is $(2L^2R/\lambda k)$-Lipschitz stable with respect to $d_\F$.
\end{lemma}
In other words, as with the standard notion of replacement stability, the Lipschitz stability depends on the Lipschitz constant of the loss, the expected norm over $\D$, and (inversely on) the degree of regularization.
\begin{proof}
We follow the proof of replacement stability of RKHS from \cite{agarwalnotes} closely.
For notational convenience, we drop reference to $\F$ in $\norm{\cdot}$ and $\langle \cdot,\cdot\rangle$.

Suppose $S \in \Z^{k-1}$ and suppose $z,z' \in \Z$ are two points such that $z = (x,y)$ and $z' = (x',y)$; i.e.\ they share the same label.
By the definition of $d_\F$, this is the only case we need to consider.
Let $f,f' \in \F$ denote the empirical risk minimizers over $S \cup \set{z}$ and $S \cup \set{z'}$, respectively.
\begin{gather*}
f = \argmin_{g \in \F} \set{\er_{S\cup \set{z}}(g) + \frac{\lambda}{2}\norm{g}^2}\\
f' = \argmin_{g \in \F} \set{\er_{S\cup \set{z'}}(g) + \frac{\lambda}{2}\norm{g}^2}\\
\end{gather*}
For $\alpha \in [0,1]$, let
\begin{gather*}
f_\alpha = \alpha \cdot f + (1-\alpha)\cdot f'\\
f_\alpha' = (1-\alpha)\cdot f + \alpha \cdot f'.
\end{gather*}
By the assumption that $\ell$ is convex, we can derive the following inequalities for any $(x,y) \in \Z$ and any $S \subseteq \Z$.
\begin{gather}
\ell(f_\alpha(x), y) \le \alpha \cdot \ell(f(x), y) + (1-\alpha) \cdot \ell(f'(x), y) \notag \\
\implies \er_S(f_\alpha) \le \alpha \cdot \er_S(f) + (1-\alpha) \cdot \er_S(f')\label{er:convex}
\end{gather}

Note that $f_\alpha$ and $f'_\alpha$ are also feasible hypotheses in the Hilbert space.
Thus, by the fact that $f,f'$ are ERMs,
\begin{align*}
\er_{S \cup \set{z}}(f) + \frac{\lambda}{2}\norm{f}^2 &\le \er_{S \cup \set{z}}(f_\alpha) + \frac{\lambda}{2}\norm{f_\alpha}^2\\
\er_{S \cup \set{z'}}(f') + \frac{\lambda}{2}\norm{f'}^2 &\le \er_{S \cup \set{z'}}(f_\alpha') + \frac{\lambda}{2}\norm{f_\alpha'}^2
\end{align*}
Rearranging, and applying convexity, we derive the following inequality.

\begin{gather}
\label{ineq:rkhs:main}
\frac{\lambda}{2}\cdot\left(\norm{f}^2 + \norm{f'}^2 - \norm{f_\alpha}^2 - \norm{f_\alpha'}^2 \right)
\le \er_{S \cup \set{z}}(f_\alpha) - \er_{S \cup \set{z}}(f) + \er_{S \cup \set{z'}}(f_\alpha') - \er_{S \cup \set{z'}}(f') 
\end{gather}
We can simplify the inner term of the LHS of (\ref{ineq:rkhs:main}) as follows.
\begin{align*}
\norm{f}^2 + \norm{f'}^2 - \norm{f_\alpha}^2 - \norm{f_\alpha'}^2 &=
\norm{f}^2 + \norm{f'}^2 - \norm{f' - \alpha (f' - f)}^2 - \norm{f + \alpha (f' - f)}^2\\
&= 2\alpha \langle f' - f, f' - f \rangle - 2\alpha^2 \norm{f' - f}^2\\
&= 2\alpha(1-\alpha) \norm{f-f'}^2
\end{align*}

Then, we can bound the RHS of (\ref{ineq:rkhs:main}) as follows.
\begin{align}
\er_{S \cup \set{z}}(f_\alpha) &- \er_{S \cup \set{z}}(f) + \er_{S \cup \set{z'}}(f_\alpha') - \er_{S \cup \set{z'}}(f') \notag\\
&\le (1-\alpha) \cdot \left(\er_{S \cup \set{z}}(f') - \er_{S \cup \set{z}}(f) + \er_{S \cup \set{z'}}(f) - \er_{S \cup \set{z'}}(f')\right) \label{appconvex}\\
&= \frac{1-\alpha}{k}\cdot \left(\ell(f'(x),y)-\ell(f(x),y) + \ell(f(x'),y)-\ell(f'(x'),y) + (k-1)\cdot 0\right) \label{zero}\\
&= \frac{(1-\alpha)}{k}\cdot \left(\ell(f'(x),y) - \ell(f'(x'),y) - \ell(f(x), y) + \ell(f(x'),y)\right)\notag\\
&\le \frac{(1-\alpha) L}{k}\cdot \langle f' - f, \varphi(x) - \varphi(x') \rangle \label{lipschitz}\\
&\le \frac{(1-\alpha) L}{k}\cdot \norm{f-f'} \cdot \norm{\varphi(x) - \varphi(x')} \label{cauchy}
\end{align}
where (\ref{appconvex}) follows by expanding according to (\ref{er:convex}), and then simplifying;
(\ref{zero}) follows by the fact that $\er_{S \cup\set{z}}(f) = \frac{1}{k}\cdot \ell(f(x),y) + \frac{k-1}{k} \er_S(f)$, but the errors on $S$ cancel to $0$, $\er_S(f') - \er_S(f) + \er_S(f) - \er_S(f')$;
(\ref{lipschitz}) follows by the Lipschitzness of $\ell$;
and (\ref{cauchy}) follows by Cauchy-Schwarz.

The analysis above applied for any $\alpha \in [0,1]$; taking $\alpha = 1/2$ implies
\begin{gather*}
\frac{\lambda}{4}\cdot\norm{f-f'}^2 \le \frac{L}{2k} \cdot \norm{f-f'}\cdot \norm{\varphi(x)-\varphi(x')}\\
\implies \norm{f - f'} \le \frac{2L}{\lambda k}\cdot d_\F(z,z').
\end{gather*}

Because $\ell$ is $L$-Lipschitz, we obtain Lipschitz-stability of $\pot_\F$.
\begin{align*}
\pot_\F(S \cup \set{z}) - \pot_\F(S \cup \set{z'}) &= \E_{(x_0,y_0) \sim \D}\left[\ell(f'(x_0),y_0) - \ell(f(x_0,y_0)\right]\\
&\le \E\left[\norm{\varphi(x_0)}\right] \cdot \frac{2L^2}{\lambda k}\cdot d_\F(z,z')\\
&= \frac{2L^2R}{\lambda k}\cdot d_\F(z,z')
\end{align*}
\end{proof}

\section{Estimating  Distributional Shapley Values -- Analysis}
\label{app:alg}

We require the following standard concentration inequality.
\begin{theorem*}[Hoeffding's Inequality]
Suppose $X_1,\hdots,X_T$ are independent random variables, where $X_t$ is bounded in the range $[-b_t,b_t]$.  Let $\overline{X} = \frac{1}{T} \sum_{t=1}^T X_t$.  Then,
\begin{equation*}
\Pr\left[\card{\overline{X} - \E[\overline{X}]} > \eps\right] \le 2 \cdot \exp\left(\frac{-T^2\eps^2}{2\cdot \sum_{t=1}^T b_t^2}\right)
\end{equation*}
\end{theorem*}

\subsection{Iteration complexity of Algorithm~1.}

\begin{theorem*}[Restatement of Theorem~3.1]
Fixing a potential $U$ and distribution $\D$, and $Z \subseteq \Z$, suppose  $T \ge \Omega\left(\frac{\log(\card{Z}/\delta)}{\eps^2}\right)$.
Algorithm~1 produces unbiased estimates and with probability at least $1-\delta$,
$\card{\val(z;\pot,\D,m) - \val_T(z)} \le \eps$.
for all $z \in Z$.
\end{theorem*}
\begin{proof}
The theorem follows from the analysis of Theorem~\ref{thm:general:biased}, by taking the stability to be trivial, $\beta(k) = 1$, and using uniform sampling $w_k = 1/m$ for all $k \in [m]$.
\end{proof}

\subsection{Running time analysis under stability.}

\begin{theorem}[Generalizes Theorem~\ref{thm:biased}]
\label{thm:general:biased}
Suppose $\pot$ is $\beta(k)$-deletion stable and for all sets $S \subseteq \Z$ of cardinality $\card{S} = k$, $\pot(S)$ can be evaluated in time $R(k)$;
consider a set of positive weights $\set{w_k : k \in [m]}$ and $p \in [0,1]$.
Algorithm~2 produces unbiased estimates of $\val(z;\pot,\D,m)$ that with probability at least $1-\delta$ are $\eps$-accurate for all $z \in Z_p$ and runs in expected time
\begin{align*}
RT_w(m) &\le O\left(p \cdot \card{Z} \cdot \frac{\log(\card{Z}/\delta)}{\eps^2m^2} \cdot \left(\sum_{k=1}^m \frac{\beta(k)^2}{w_k}\right) \cdot \left(\sum_{k=1}^m w_k \cdot R(k)\right)\right)
\end{align*}
\end{theorem}

\begin{proof}
First, we bound the iteration complexity and time complexity to evaluate models at each iteration within Algorithm~2.
The running time bound then follows by the fact that we need to evaluate a model per $z \in Z_p$, per iteration where the expected cardinality of $\card{Z_p} = p \cdot \card{Z}$.

Abusing notation, denote by
\begin{align*}
\Delta_z\pot(k) = \E_{S \sim \D^{k-1}}\left[\pot(S \cup \set{z}) - \pot(S)\right].
\end{align*}
Suppose we sample $k$ according to a possibly non-uniform discrete distribution where $\Pr[k \in m] = w_k$; we denote a random drawn from this distribution by $k \sim [m]_w$.
Then, by sampling $k \sim [m]_w$, computing $\Delta_z\pot(S)$ for $S \sim \D^k$, and reweighting, we obtain an unbiased estimate of $\val(z;\pot,\D,m)$.
\begin{align*}
\val(z;\pot,\D,m) &= \E_{k \sim [m]}\left[\Delta_z\pot(k)\right]\\
&= \frac{1}{m}\sum_{k=1}^m \Delta_z\pot(k)\\
&= \sum_{k=1}^m w_k \frac{\Delta_z\pot(k)}{w_k m}\\
&= \E_{k \sim [m]_w}\left[\frac{\Delta_z\pot(k)}{w_k m}\right]
\end{align*}

For simplicity, we analyze a sampling scheme where we sample $T_k$ sets with cardinality $k$ for $T_k \ge w_k \cdot T$.
(By the multiplicative Chernoff bound, this event will occur with high probability.)
That is, for all $z \in Z$ and for each $k \in [m]$, we sample $T_k$ subsets $S \sim \D^k$ and compute $\Delta_z\pot(S)$.
For each $z \in Z$, each such sample is an independent unbiased estimate of $\Delta_z\pot(k)$, so reweighting according to $w_k$ and averaging over $k \in [m]$ gives an unbiased estimate of $\val(z;\pot,\D,m)$.
\begin{align*}
\val_T(z) &=
\frac{1}{T}\sum_{k=1}^{m}\sum_{t=1}^{T_k}\frac{\Delta_z\pot(S_t)}{w_k m}
\end{align*}
Note that by $\beta(k)$-deletion stability, for each $k$, the terms in the summation associated with $\Delta_z\pot(k)$ are bounded in magnitude by $\frac{\beta(k)}{w_k m}$. 
Thus, we can apply Hoeffding's inequality to derive the following bound to obtain $\eps$-error with probability at least $1-\delta_0$.
\begin{align*}
\delta_0 &\ge 2\cdot \exp\left(\frac{-\eps^2T^2}{2\cdot \sum_{k=1}^m\sum_{t=1}^{T_k} \left(\frac{\beta(k)}{w_k m}\right)^2}\right)\\
&\ge 2\cdot \exp\left(\frac{-\eps^2T^2}{\frac{2}{m^2}\cdot \sum_{k=1}^m w_k \cdot T \cdot \left(\frac{\beta(k)}{w_k}\right)^2}\right)\\
&= 2\cdot \exp\left(\frac{-\eps^2m^2T}{2\cdot \sum_{k=1}^m \frac{\beta(k)^2}{w_k}}\right)
\end{align*}
Thus, taking the failure probability $\delta_0 = \delta/\card{Z}$ small enough to union bound over all $z \in Z$, we derive the following bound on $T$.
\begin{equation*}
T \ge \Omega\left(\frac{\log(\card{Z}/\delta)}{\eps^2m^2} \cdot \sum_{k=1}^m \frac{\beta(k)^2}{w_k}\right)
\end{equation*}

Using this bound on $T$, we can compute the necessary running time for Algorithm~2 in terms of $R(k)$ per $z \in Z_p$.
\begin{align*}
T_w(m) &= T \cdot \sum_{k=1}^m w_k \cdot R(k)\\
&= \frac{\log(\card{Z}/\delta)}{\eps^2m^2} \cdot \left(\sum_{k=1}^m \frac{\beta(k)^2}{w_k}\right) \cdot \left(\sum_{k=1}^m w_k \cdot R(k)\right)
\end{align*}
Thus, we can compare various sampling schemes for different stability factors.
Note that in the case of the uniform sampling scheme, where $w_k = 1/m$ for all $k \in [m]$, the sampling probabilities cancel.
\begin{align*}
T_u(m) &= \frac{\log(\card{Z}/\delta)}{\eps^2m^2} \cdot \left(\sum_{k=1}^m \frac{\beta(k)^2}{1/m}\right) \cdot \left(\sum_{k=1}^m 1/m \cdot R(k)\right)\\
&= \frac{\log(\card{Z}/\delta)}{\eps^2m^2} \cdot \left(\sum_{k=1}^m \beta(k)^2\right) \cdot \left(\sum_{k=1}^m R(k)\right)
\end{align*}
Thus, the overall running time is given as
$RT_w(m) = p \cdot \card{Z} \cdot \frac{\log(\card{Z}/\delta)}{\eps^2m^2} \cdot \left(\sum_{k=1}^m \frac{\beta(k)^2}{w_k}\right) \cdot \left(\sum_{k=1}^m w_k \cdot R(k)\right)$.
\end{proof}

Concretely, to see the special case of the theorem stated as Theorem~3.2, suppose that $\beta(k) = k^{-b}$ for $b \ge 1/2$ and $R(k) = k^c$ for $c \ge 1$.
With these settings of the parameters, uniform sampling takes time
\begin{align*}
T_u(m) &= \frac{\log(\card{Z}/\delta)}{\eps^2m^2} \cdot \left(\sum_{k=1}^m k^{-2b}\right) \cdot \left(\sum_{k=1}^m k^{c}\right)\\
&= \frac{\log(\card{Z}/\delta)}{\eps^2m^2} \cdot O(\log(m)) \cdot O(m^{c+1})\\
&\le \frac{\log(\card{Z}/\delta)}{\eps^2} \cdot \tilde{O}\left(m^{c-1}\right).
\end{align*}
Suppose, instead, we take $w_k \propto k^{1-2b}$; that is, we choose $w_k$ such that the first summation will still be bounded by $H_m = \Theta(\log(m))$.
Under such a sampling scheme, the running time is bounded as
\begin{align*}
T_w(m) &= \frac{\log(\card{Z}/\delta)}{\eps^2m^2} \cdot \left(\sum_{k=1}^m k^{-1}\right) \cdot \left(\sum_{k=1}^m k^{c+1-2b}\right)\\
&= \frac{\log(\card{Z}/\delta)}{\eps^2m^2} \cdot O(\log(m)) \cdot O(m^{c+2-2b})\\
&\le \frac{\log(\card{Z}/\delta)}{\eps^2} \cdot \tilde{O}\left(m^{c-2b}\right).
\end{align*}
In other words, the biased sampling scheme allows us to save roughly a factor-$m^{2b-1}$ in computation time.
Thus, if $\pot$ is $O(1/k)$-deletion stable, then the biased sampling scheme saves roughly a factor $m$ in computation time.

\subsection{Finite Sample Approximation to $\D$}

While the analysis of Theorem~\ref{thm:general:biased} focuses on the running time of Algorithm 2, we can equally interpret it as a sample complexity bound.
In particular, taking $R(k) = k$ corresponds to the sample complexity of taking a fresh sample $S \sim \D^k$ per iteration.
Thus, using Algorithm~2, the naive sample complexity given by resampling for each iteration gives the bound of
\begin{align*}
M &\approx \frac{\log(\card{Z}/\delta)}{\eps^2m^2}\cdot \left(\sum_{k=1}^m \frac{\beta(k)^2}{w_k}\right) \cdot \left(\sum_{k=1}^m w_k \cdot k\right)
\end{align*}
Taking $\beta(k) = w_k = 1/k$ yields
\begin{align*}
M &\approx \frac{\log(\card{Z}/\delta)}{\eps^2m}\cdot \left(\sum_{k=1}^m \frac{1}{k}\right)\\
&\approx \frac{\log(m) \cdot \log(\card{Z}/\delta)}{\eps^2m}
\end{align*}

\clearpage
\section{Additional Performance Experiments}
\label{app:exp}

We report the results of additional empirical evaluations of Algoirthm~\ref{alg:fast}, as described in Section~\ref{sec:alg:empirical}.
We depict the results of point removal experiments by plotting model performance vs.\ fraction of training data removed, where points are removed in order of their Shapley estimates or randomly.
We apply the interpolation and weighted sampling speed-ups from Algorithm~\ref{alg:fast}, independently.
Both strategies obtain an order-of-magnitude speed-up with no qualitative performance change.
\begin{figure*}[ht!]
\includegraphics[width=0.79\linewidth]{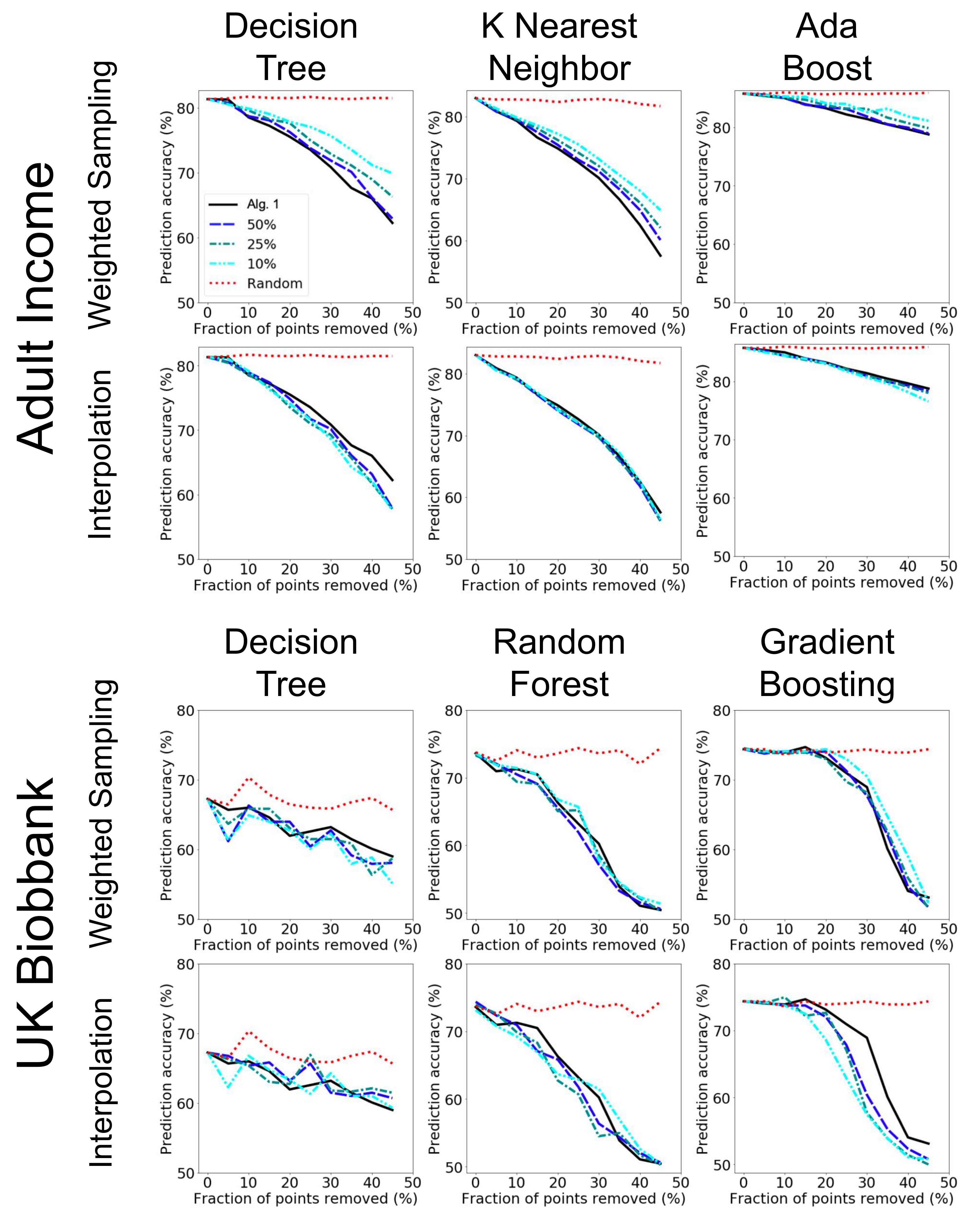} 
\centering
\caption{
Point removal curves for $\D$-Shapley estimates under various speed-ups, using weighted sampling and interpolation
across two ML tasks and three learning algorithms.
\label{fig:speedup_app}}
\end{figure*}

\clearpage
\subsection{Speeding-up Distributional Shapley for Cifar10}

In this experiment, we apply both speed-up methods to compute the distributional Shapley values for CIFAR10 dataset.
We apply weighted sampling corresponding to a speed-up factor of $10$ and subsampling with interpolation corresponding to a speed-up factor of $50$, to compute values for $1000$ data points.
We valuate points based on their effect on an image classification task.
We use an Inception-v3~\cite{szegedy2016inception} model, pretrained on the ILSVRC2012 (Imagenet)~\cite{russakovsky2015imagenet} dataset; then,
we base our potential function $\pot(S)$ on the performance of the model resulting after retraining the final layer of the network (holding all other layers fixed) using the retraining set $S$.
Figure~\ref{fig:cifar} shows the point-removal results for the complete dataset (e.g. removing $50\%$ of the points with the highest $\D$-Shapley value causes the prediction accuracy to drop from $77\%$ to $68\%$.)

\begin{figure}[ht]
\centering
\includegraphics[width=0.35\linewidth]{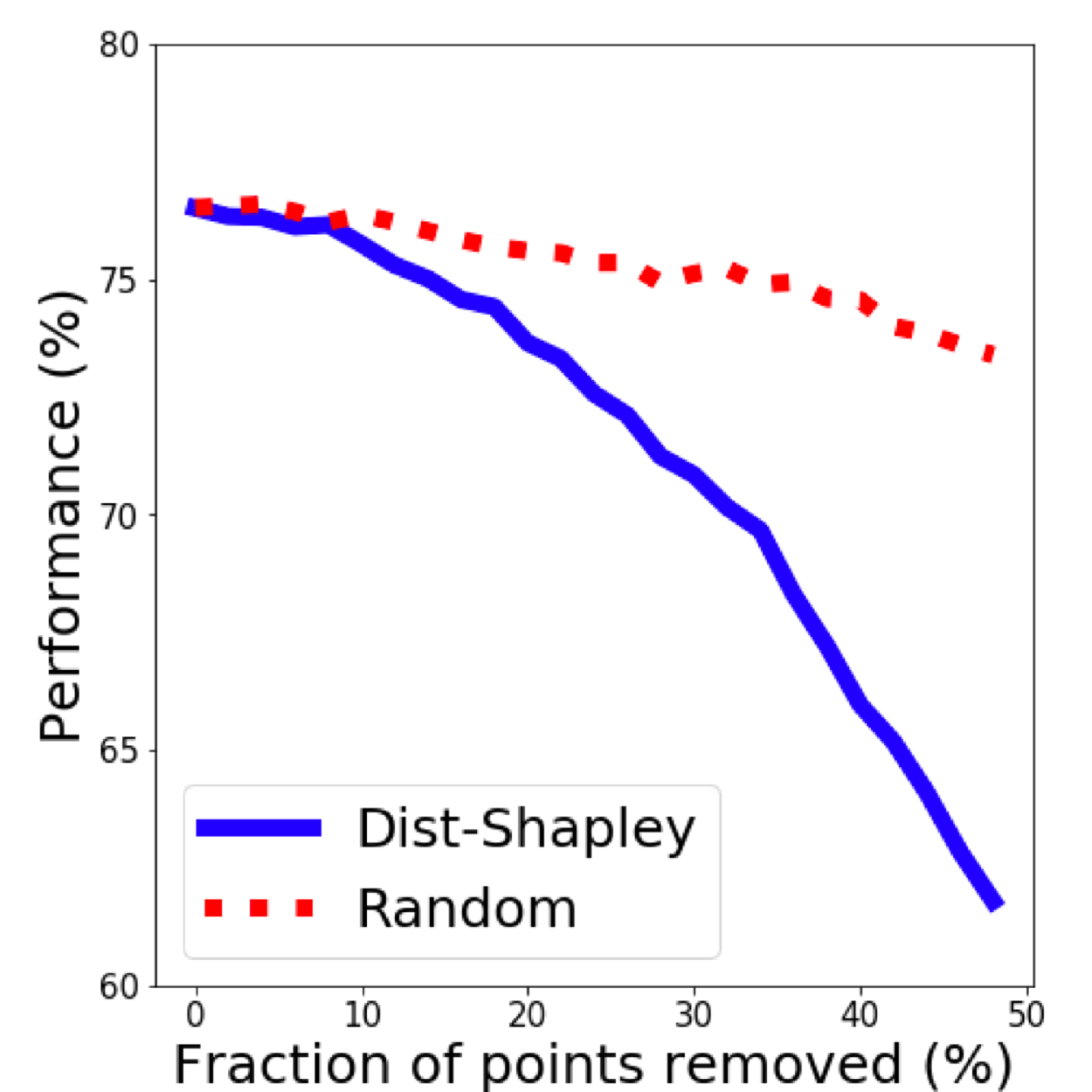} 
\caption{Point removal experiment for CIFAR10, using distributional Shapley estimates computed by \textsc{Fast-$\D$-Shpaley}.
\label{fig:cifar}}
\end{figure} 
\clearpage
\section{Additional Case Study Experiments}
\label{app:case}

We report the complete results for the caze study from Section~\ref{sec:case}.
We use four large-scale datasets from the UCI repository~\cite{Dua:2019}:
\begin{enumerate}[(1)]
    \item Covertype dataset with $581,012$ samples where each sample contains $54$ visual features of forest images and the task is to detect the type of forest cover (from a set of $7$ different covers). We use a Random Forest model.
    \item Diabetes130 dataset~\cite{strack2014impact} that with $100,000$ samples. Each sample contains $54$ patient and hospital features. The task is to predict whether the patient will be readmitted to the hospital. We use an AdaBoost model.
    \item Wearable Computing: Classification of Body Postures and Movements (PUC-Rio) Data Set which has $165,632$ points where each point has $18$ attributes and has one of the $5$ postures. We use a multinomial logistic regression model.
    \item Dataset for Sensorless Drive Diagnosis Data Set that contains $58,509$ data points. Each data point has $48$ features. The dataset has $11$ classes. We use a Gradient Boosting model.
\end{enumerate}
\begin{figure*}[ht]
\includegraphics[width=0.75\linewidth]{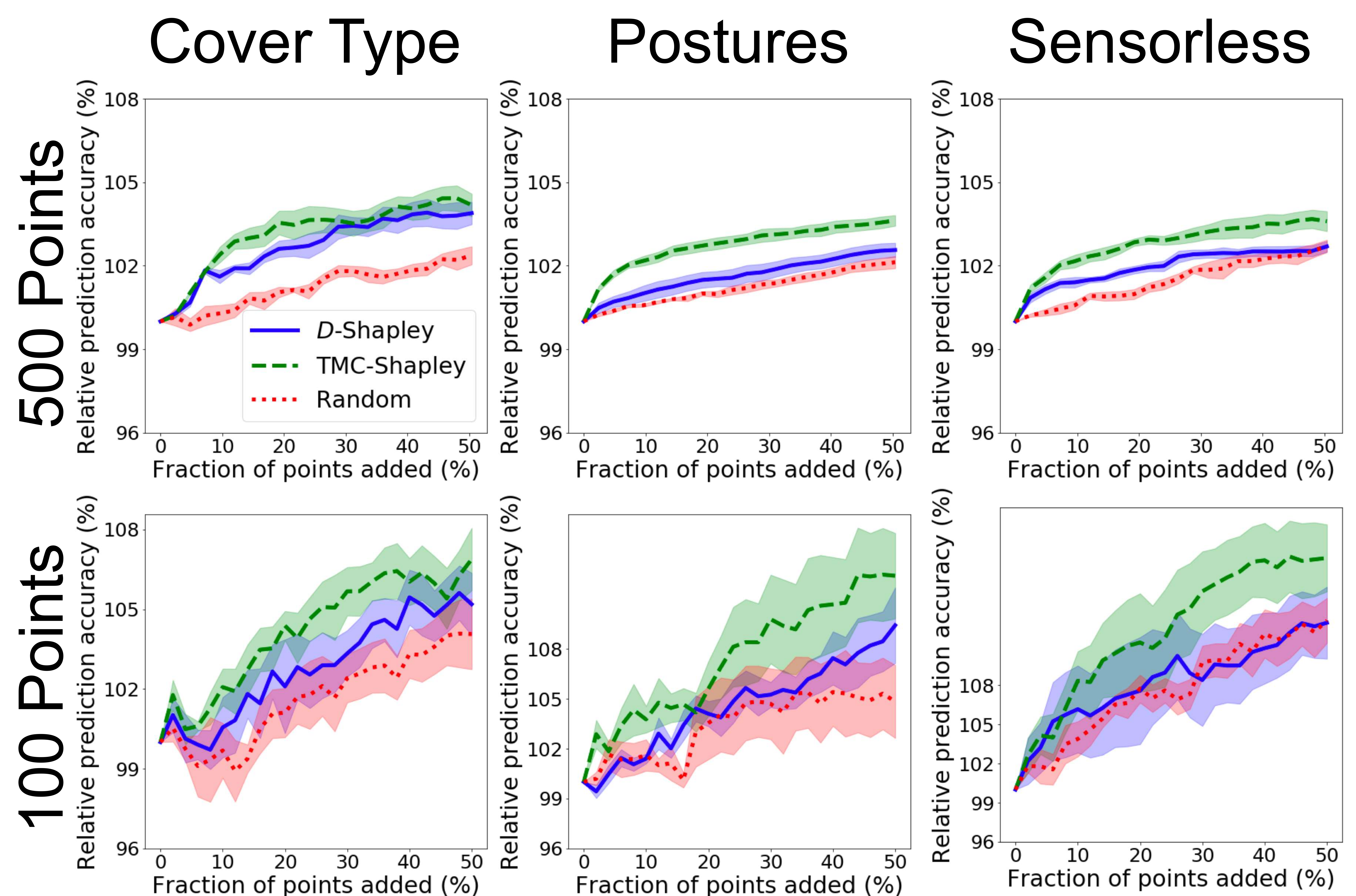} 
\centering
\caption{Points from an acquired set are added to the buyer's initial dataset in three different orders: according to $\val$ ($\D$-Shapley), according to $\sh$ (TMC), and randomly.
The plot shows the change in the accuracy of the model, relative to its performance using the buyer's initial dataset, as the points are added; shading indicates standard error of the mean. 
\label{fig:selling_app}}
\end{figure*}

\end{document}